\documentclass{article}

\usepackage{arxiv}

\usepackage[utf8]{inputenc} % allow utf-8 input
\usepackage[T1]{fontenc}    % use 8-bit T1 fonts
\usepackage{hyperref}       % hyperlinks
\usepackage{url}            % simple URL typesetting
\usepackage{booktabs}       % professional-quality tables
\usepackage{amsfonts}       % blackboard math symbols
\usepackage{nicefrac}       % compact symbols for 1/2, etc.
\usepackage{microtype}      % microtypography
\usepackage{lipsum}		% Can be removed after putting your text content
\usepackage{algorithm}
\usepackage{algorithmic}
\usepackage{subcaption} % for subfigure environment
\usepackage{graphicx}
\usepackage{natbib}
\usepackage{doi}
\usepackage{booktabs}
\usepackage{amsthm}
\usepackage{amsmath}
\usepackage{amsfonts}
\usepackage{subcaption}
\usepackage{xr}
\usepackage{float}

\title{GaP-MoE : Explicit Expert Specialization in Mixture-of-Experts via Orthogonal Gradient Accumulation and Projection}

\author{ {\hspace{1mm}Donghyun~Jeong} \\
	Department of Statistics and Data Science\\
	University of Seoul\\
	Republic of Korea\\
	\texttt{dandong00@uos.ac.kr} \\
	\And
	{\hspace{1mm}Yoon-Yeong~Kim}\thanks{Corresponding Author} \\
	Department of Statistics\\
	University of Seoul\\
	Republic of Korea\\
	\texttt{yykim@uos.ac.kr} \\
}

\renewcommand{\shorttitle}{GaP-MoE}

\hypersetup{
pdftitle={GaP-MoE},
pdfsubject={q-bio.NC, q-bio.QM},
pdfauthor={Donghyun Jeong, Yoon-Yeong Kim},
pdfkeywords={Mixture-of-Experts, Expert Specializatoin, Representation Collapse},
}

\begin{document}
\maketitle
\setcounter{footnote}{0}

\begin{abstract}
Sparse Mixture-of-Experts (SMoE) models improve scalability by routing input tokens to only a subset of expert networks, but they often suffer from representation collapse, in which multiple experts converge toward similar functions and lose their intended diversity. Prior work has addressed this limitation either by improving router architectures or by introducing auxiliary loss terms that implicitly penalize similarity among experts. However, router-based approaches remain vulnerable to routing collapses, and loss-based approaches can conflict with the task's original objective during optimization. In this work, we propose GaP-MoE (Gradient Accumulation and Projection in MoE), which induces expert specialization by directly manipulating expert output gradients rather than modifying the loss function. At each training step, GaP-MoE stochastically selects one of the $TopK$ routed experts as a leader based on normalized routing scores, and projects the gradients of the remaining follower experts onto the orthogonal complement of the subspace spanned by the leader's recent gradients. This removes directional overlap between expert gradients without altering the task loss, allowing specialization to emerge as a direct consequence of the optimization dynamics rather than an externally imposed constraint. We evaluate GaP-MoE in both fine-tuning settings (BERT-based, on GLUE) and from-scratch pretraining settings (GPT-2-based, on WikiText-103, enwik8, and text8), and further confirm its compatibility as a drop-in module for existing SMoE architectures. Experimental results show that GaP-MoE achieves performance comparable to or better than standard SMoE.
\end{abstract}

% keywords can be removed

\section{Introduction}
With the development of Transformer-based models \citep{trans}, numerous deep learning models have made remarkable progress across various tasks in recent years. The simplest way to improve the performance of deep learning models for each task is to follow the so-called “scaling law,” which involves increasing both the model size and the dataset size \citep{scaling01, gpt3}. However, this also implies an increase in computational cost.

A commonly used approach to address this issue is the Sparse Mixture-of-Experts (SMoE) \citep{shazeer2017}. Instead of a single neural network, SMoE uses multiple independent networks called \textit{experts} to increase the overall model capacity \citep{MoE}, and applies \textit{$TopK$ routing}, which selects the $K$ experts best suited to process each input token. Each selected expert processes the token using its own parameters, and the router combines their outputs as a weighted sum using the corresponding routing scores, which were originally used for \textit{$TopK$ routing}.

However, SMoE does not inherently guarantee that different experts learn distinct functions. For example, if the router collapses, tokens may be assigned disproportionately to only a small subset of experts or distributed without meaningful expert preference, preventing the experts from receiving sufficiently differentiated training signals \citep{krishnamurthy2023improving}. Even when the router does not collapse, experts may remain insufficiently specialized if they observe similar inputs or evolve from similar initialization conditions \citep{dai2024deepseekmoe}.
Therefore, regardless of whether insufficient specialization originates from the routing process or the optimization dynamics of the experts, encouraging expert specialization remains important for fully utilizing the capacity of SMoE.

Consequently, various studies have been conducted to promote expert specialization. Some approaches identify low-dimensional degeneration of expert representations as a cause of insufficient specialization and encourage low-dimensional operations at the router level \citep{chi2022representation}.
Other approaches dynamically adjust the number of experts used for routing \citep{pham2024competesmoe}, efficiently measure and regulate parameter similarity \citep{do2025simsmoe}, or directly orthogonalize token output representations \citep{liu2024diversifying}.
More recently, this line of research has extended to gradient regulation, where auxiliary loss functions are introduced to prevent gradient conflicts among tokens within the same expert \citep{yang2025solving, guo2026advancing}.

However, existing specialization methods have potential drawbacks. Router-based methods primarily improve routing decisions and encourage specialization indirectly by modifying the geometry or distribution of token assignments, rather than directly constraining how the selected experts learn their representations \citep{wang2026myth}.
Methods that regulate expert parameters or representations through auxiliary loss functions can more directly induce specialization. However, such auxiliary objectives may conflict with the original task loss used to train the model parameters, potentially hindering both learning processes \citep{wang2024auxiliary, guo2026advancing}.

Accordingly, we propose Gradient Accumulation and Projection in Mixture-of-Experts, or \textbf{GaP-MoE}, a learning algorithm that retains the standard training procedure of SMoE while explicitly applying orthogonality constraints to already-computed expert output gradients.
Rather than introducing an additional loss term, GaP-MoE directly modifies redundant gradient directions after the task gradients have been generated, thereby preserving the original task objective while promoting expert specialization.

Specifically, we store the recent output gradients received by each expert as an orthonormal basis using the Gram–Schmidt process. Among the $TopK$ selected experts, one expert is probabilistically designated as a leader according to the normalized routing scores, while the remaining experts are treated as followers. The leader retains its original task gradient, whereas each follower removes the component of its output gradient that overlaps with the subspace spanned by the leader’s accumulated gradients. In this way, GaP-MoE encourages co-routed experts to follow differentiated update directions without requiring additional forward or backward passes.
% To the best of our knowledge, no method exists that directly enforces the orthogonality of each expert’s output gradients.
% In this study, we store the gradients received by each expert in the form of an orthonormal basis using the Gram-Schmidt method, and simultaneously achieve orthogonal training and expert specialization by probabilistically selecting experts with relatively low expertise from the $TopK$ and removing their gradients.
The key benefits of GaP-MoE are as follows.
\begin{itemize}
\item Explicitly induces expert specialization through gradient orthogonalization.
\item Reuses already computed expert output gradients, avoiding additional forward/backward passes and gradient recomputation.
\item Its simple structure ensures compatibility with various existing router architectures and auxiliary loss functions.
\end{itemize}

\section{Related Works}

\subsection{Sparse Mixture-of-Experts}
Following the conventional SMoE formulation \citep{shazeer2017}, a Mixture-of-Experts module consists of a set of $N$ neural networks, denoted as $E_1,\ldots, E_N$, each with independent parameters. It is generally used to replace the FFN layer in Transformer models, in which each expert is typically implemented as a two-layer MLP \citep{lepikhin2020gshard, fedus2022switch}
\begin{equation}
y_i(x) = W^2_i(\sigma(W^1_i x + b^1_i)) + b^2_i, i \in {1,\ldots, N} \notag
\end{equation}
where $x$ is the input token, and $(W^1_i, W^2_i)$ and $(b^1_i, b^2_i)$ denote the weights and biases of the $i$-th expert. However, this form is not fixed, and various types of experts and corresponding architectures are being studied \citep{zhang2023robust,zhang2022mixture}.
SMoE additionally employs a gating network, or router, $G(x)$ that computes a routing score for each expert w.r.t. token $x$. Based on these scores, the router selects the $K$ highest-scoring experts and retains their scores, while suppressing the scores of all unselected experts to zero. %\cite{jiang2024mixtral,dai2024deepseekmoe,yang2025qwen3}).
% Accordingly, $G(x)_i>0$ if the $i$-th expert is included in the $TopK$ set, and $G(x)_i=0$ otherwise.
% \begin{equation}
% G(x)_i = TopK(\mathrm{SoftMax}((W_g x)_i), K) \notag
% \end{equation}
% Here, $W_g$ is the routing parameter matrix, and $K$ is a hyperparameter.
% The router determines how well each expert can process the input token $x$. It learns the routing parameters to output logits and uses the $TopK$ operation to decide which expert to assign the token to (in more recent open-weight MoE LLMs, the $TopK$ operation is commonly applied before softmax \cite{jiang2024mixtral,dai2024deepseekmoe,yang2025qwen3}). %Next, the routing score $G(x)_i$ for the selected $i$-th expert is computed using softmax.
The token is then passed to the selected expert, and their outputs are combined using the corresponding routing scores as  %where a forward pass is performed on each neural network, and the results are linearly combined using the respective routing scores $G(x)_i$ that were previously computed.
\begin{equation}
SMoE(x) = \sum_{i=1}^{N}G_i(x) y_i (x) \label{eq_moe}
\end{equation}
Since $G_i(x) = 0$ for experts not included in the $TopK$, only a subset of experts is activated for each token, thereby achieving sparse computation.

\subsection{Expert Specialization and Representation Collapse}

% \textit{Representation collapse}, in which hidden representations degenerate into linear combinations of expert embeddings and become confined to a low-dimensional subspace, was first identified in \cite{chi2022representation} as a major obstacle to learning distributed expert representations in MoE models.

\citet{chi2022representation} describe representation collapse as the degeneration of representations during neural network training and investigate this phenomenon in sparse Mixture-of-Experts (SMoE) models. Their analysis suggests that routing-induced updates can constrain token hidden representations toward a low-dimensional subspace spanned by expert embeddings and draw representations of tokens routed to the same expert closer together, potentially limiting their expressive capacity.

Subsequent studies have used the term \textit{representation collapse} more broadly to describe several forms of insufficient differentiation among experts. These include cases in which multiple experts learn similar representations, thereby limiting specialization, and cases in which the router assigns most tokens to only a small subset of experts  \citep{pham2024competesmoe,do2025simsmoe,do2023hyperrouter}.
Research that aims to make individual experts learn distinct representations or functions is generally studied under the notion of \textit{expert specialization}  \citep{krishnamurthy2023improving,dai2024deepseekmoe}.

Recent studies have employed various approaches to promote expert specialization and alleviate representation collapse. A prominent example is router-based methods, which prevent tokens from being concentrated on specific experts and encourage more differentiated routing decisions. XMoE  \citep{chi2022representation} maps routing representations to a low-dimensional space to mitigate representational degeneration.
 \citep{nguyen2025statistical} further analyzed how perturbing a cosine router, which has recently emerged as an alternative to the conventional linear router, can help prevent representation collapse. 

Other studies address representation collapse by directly regulating expert weights, outputs, or gradients. \citet{guo2026advancing} promote expert specialization using a loss that orthogonalizes expert parameters, alongside a loss that promotes more diverse routing. SimSMoE \citep{do2025simsmoe} induces differentiation by reducing the similarity between the output representations produced by different experts. SD-MoE \citep{huang2026sd} decomposes expert and gradient matrices via SVD into dominant and expert-specific components, facilitating the learning of an expert-specific subspace.

However, existing methods encourage expert outputs to explore diverse subspaces either implicitly through routing mechanisms or explicitly through additional loss functions. %either do not explicitly encourage each expert's output to explore diverse subspaces or attempt to achieve this indirectly through the loss function.
In contrast, our method, GaP-MoE, directly differentiates the update directions of expert outputs by projecting follower gradients away from the recent gradient subspace of a dynamically selected leader. This promotes expert specialization without introducing an additional optimization objective, while preserving the original task gradient of the selected leader.
% experts to learn more diverse representations through gradient differentiation of expert outputs, while simultaneously probabilistically selecting the expert best suited to process a given token—a process that does not interfere with that expert’s learning.
\begin{figure*}[t]
    \centering

    \begin{subfigure}{0.495\linewidth}
        \centering
        \includegraphics[width=1\linewidth]{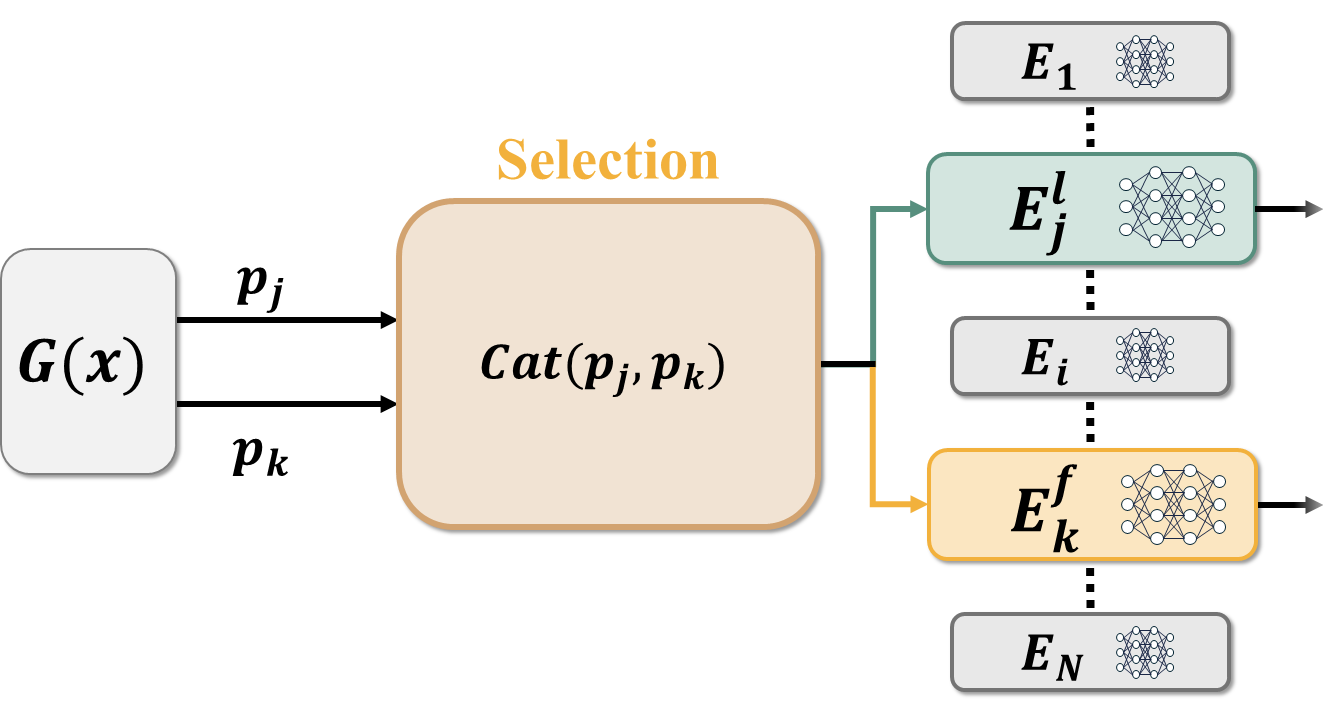}
        \caption{Dynamic Leader Selection}
        \label{fig:dyn}
    \end{subfigure}
    \begin{subfigure}{0.495\linewidth}
        \centering
        \includegraphics[width=1\linewidth]{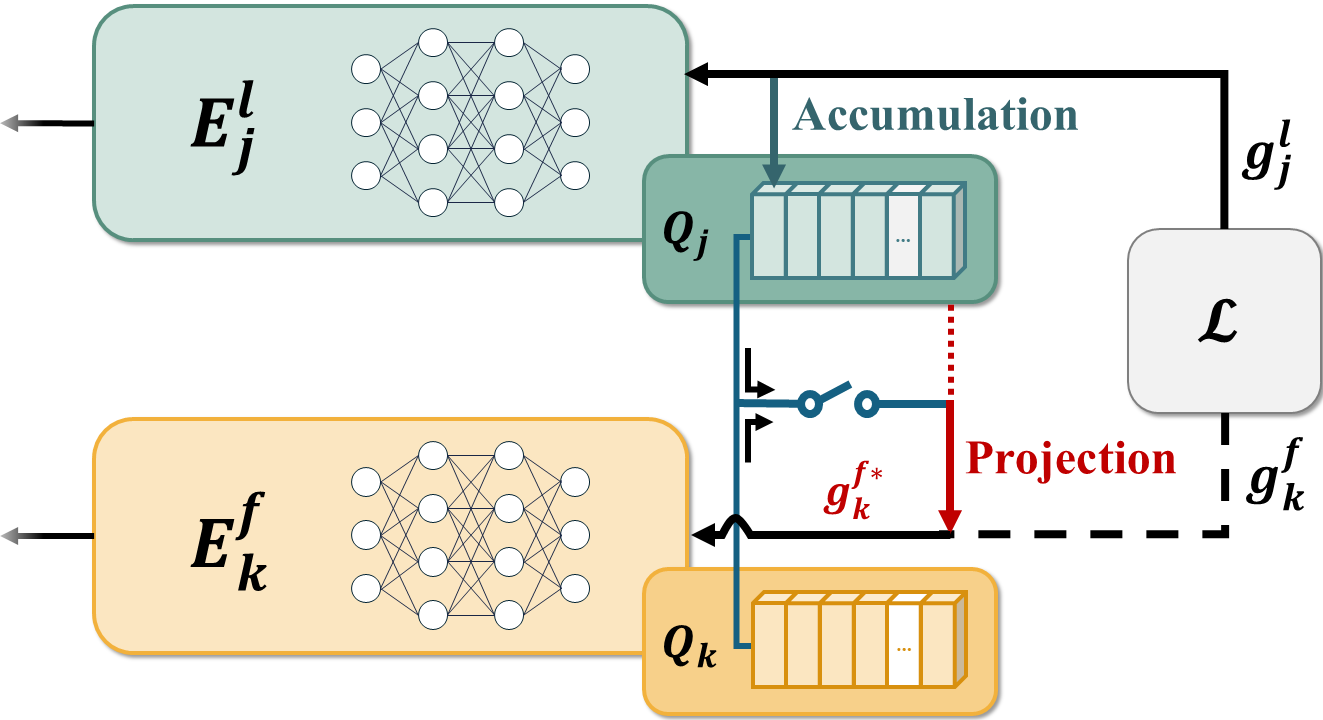}
        \caption{Gradient Accumulation and Projection}
        \label{fig:proj}
    \end{subfigure}
    \caption{(a) GaP-MoE selects a leader expert from the $TopK$ set using normalized routing probabilities $p_i$. (b) The aggregated output gradients of the leader $E_j^l$ are stored in $Q_j$ as orthonormal basis vectors. When the similarity between their most recent basis vectors exceeds the threshold $\gamma$, the follower gradient is projected onto the orthogonal complement of the leader subspace and used to update $E_k^f$.}%by using probabilities $p_i$ which were normalized from each routing logit $G(x)_i$ as the probability of being selected as the leader. The aggregation of output gradients $g_i^l, g_j^f$ calculated from the loss $\mathcal{L}$ for $E_i^l$ (leader) and $E_j^f$ (follower) is stored in the buffer queues $Q_i, Q_j$ in the form of orthonormal bases. If the cosine similarity between the most recent basis vectors in $Q_ij$ and $Q_j$ exceeds the threshold $\gamma$ (The blue line), the projection is activated(The red line), and the projected follower gradient $g_j^{f*}$ is used to update $E^f_j$.}
    \label{fig:fig1}
\end{figure*}

\begin{algorithm}[t]
\caption{GaP-MoE (short version)}
\label{alg:algorithm_short}
\textbf{Input}: mini-batch tokens $\{x_n\}$; router $G(\cdot)$ \\
\textbf{Parameter}: experts $\{\theta_i\}$; buffers $\{Q_i\}$
\begin{algorithmic}[1]
\FOR{each token $x_n$}
    \STATE Route $x_n$ to top-2 experts; sample leader $l$
    \STATE Compute expert outputs and weighted MoE output
\ENDFOR
\STATE Backpropagate to obtain gradients $g_n$ for all tokens
\STATE \textbf{Projection}: for each follower $E_k^f$, project $g_k^f$ using leader buffer $Q_j$ if $\mathrm{cossim}(b_j^{(r)}, b_k^{(r)}) > \gamma$
\STATE \textbf{Accumulation}: for each expert $i$, average its leader gradients and update $Q_i$ via Gram--Schmidt
\STATE Update $\{\theta_i\}$ with (projected) gradients
\end{algorithmic}
\end{algorithm}

\section{Proposal of GaP-MoE}

For simplicity, this section describes the method based on $Top2$ routing. GaP-MoE consists of three main components: \textit{Dynamic Leader Selection}, \textit{Gradient Accumulation}, and \textit{Gradient Projection}. %The overall process is illustrated in Figure \ref{fig:fig1}, with the complete training procedure summarized in Algorithm \ref{alg:algorithm}. 
The overall process is illustrated in Figure \ref{fig:fig1}, with the training procedure summarized in Algorithm \ref{alg:algorithm_short}. The full algorithm is provided in Appendix~\ref{appedinx:full_algorithm}.

\subsection{Dynamic Leader Selection}

For each $d_{in}$-dimensional input token $x\in \mathbb{R}^{d_{in}}$, the router computes sparse routing scores $G_i(x)$ over $N$ experts, where $G_i(x)>0$ only for the $TopK$ selected experts. GaP-MoE normalizes these nonzero routing scores and uses them as the probabilities for selecting one leader expert as
\begin{equation}
p_i = \frac{G_i(x)}{\sum_{i=1}^{N}G_i(x)}, \quad l \sim \mathrm{Categorical}(p_1,\dots,p_N).
\end{equation}
If expert $E_j$ is sampled as the leader, denoted by $E_j^l$, its original task gradient is preserved without projection. Each remaining selected expert $E_k$, denoted by $E_k^f$, becomes a follower whose output gradient is subsequently differentiated from the leader's gradient subspace through the projection described in Section \ref{sec:projection}.
% The core structure of MoE is that the router calculates a routing score $G(x)_i, i \in [1, 2, \ldots, N]$, where $N$ refers to the total number of experts, based on the input token $x \in R^{d_{in}} $and information from each expert, and then computes a weighted sum of the outputs from each expert $E_i$ using that score. GaP-MoE normalizes the routing scores so that their sum equals 1, and then uses them as the probability that each expert will become a leader. 
% \begin{equation}
% p_i = \frac{G(x)_i}{\sum_{i=1}^{N}G(x)_i}, \notag
% \end{equation}

% \begin{equation}
% l \sim \mathrm{Categorical}(p_1,\dots,p_N)
% \end{equation}

% If the $i$-th expert $E_i$ is selected as the leader($E_i^l$), its gradient is protected. And the $j$-th Expert $E_j$, which is not selected as the leader, becomes the follower ($E_j^f$).

\subsection{Gradient Projection}\label{sec:projection}
Gradient projection modifies a gradient direction by removing its component along a specified subspace. It has been widely used to mitigate gradient interference in continual and multi-task learning \citep{yu2020gradient,lopez2017gradient}.
GaP-MoE adopts this principle to reduce overlapping update directions between co-routed experts, rather than between different tasks \citep{farajtabar2020orthogonal}.

% Gradient projection is a method that directly manipulates the direction of the gradient to avoid collisions or guide the model to learn in a specific target direction. It is primarily used in continual learning and multi-task learning that handle various tasks \cite{yu2020gradient,lopez2017gradient}. 

% Specifically, our method is inspired by the orthogonal gradient projection strategy of OGD \cite{farajtabar2020orthogonal}, but differs in that it performs projection between experts rather than across tasks. In the aforementioned research and in this field, gradient projection is primarily used to prevent performance degradation by adjusting the gradients of a new task to be orthogonal to the gradient updates of the existing task. However, in GaP-MoE, it is used to explicitly constrain each expert to learn different outputs for the same input token. 

For clarity, let the output of the $i$-th expert be $y_i = \theta_i x$, where $\theta_i \in \mathbb{R}^{H \times d_{in}}$ denotes the weight matrix of the $i$-th expert with $H$ denoting the output dimension. 
The gradient with respect to the expert output is %Then, in the loss function of the MoE architecture, the gradient with respect to the expert output is expressed as
\begin{equation}
g_i=\frac{\partial \mathcal{L}}{\partial y_i},
\end{equation}
where $\mathcal{L}$ denotes the task loss for the current token\footnote{Since the $i$-th expert can either be selected as leader or follower, we use upper indices $l$ and $f$ to discriminate the selection, i.e., $g_i^l$ and $g_i^f$, respectively.}.

For a leader-follower pair $(E_j^l, E_k^f)$, GaP-MoE determines whether projection is required by comparing their most recently stored gradient basis, {$(b_j^{(r)}, b_k^{(r)})$, from the gradient buffers $Q_j$ and $Q_k$} described in Section \ref{sec:accumulation}.
Specifically, projection is activated when
\begin{equation}
\mathrm{cossim}(b_j^{(r)}, b_k^{(r)}) = \frac{b_j^{(r)}\cdot b_k^{(r)}}{\left\| b_j^{(r)} \right\| \cdot \left\|b_k^{(r)} \right\|} > \gamma, \notag
\end{equation}
where $\gamma\in[-1,1]$ is a similarity threshold. Setting $\gamma=-1$ activates projection for every leader-follower pair.
% The decision to perform gradient projection is made based on the cosine similarity of the most recent(i.e. $r$-th) gradient information $(b_i^{l,(r)}, b_j^{f,(r)})$ in the buffers of the leader $(Q^l_i)$ and the follower $(Q^f_j)$, which will be explained in Section 3.3. This allows GaP-MoE to determine whether the currently selected expert combination $(E_i^l, E_j^f)$ is being updated in the same direction. For the projection threshold $\gamma \in [-1, 1]$ set as a hyperparameter, if 
% \begin{equation}
% \mathrm{cossim}(b_i^{l,(r)}, b_j^{f,(r)}) = \frac{b_i^{l,(r)}\cdot b_j^{f,(r)}}{\left\| b_i^{l,(r)} \right\| \cdot \left\|b_j^{f,(r)} \right\|} > \gamma, \notag
% \end{equation}
% projection is performed; when $\gamma = -1$, gradient projection is always performed.
When projection is activated, the follower output gradient is modified as
% If the similarity exceeds $\gamma$, the gradient projection 
\begin{equation}
g_k^{f*} = (I - \alpha  Q_j^\top Q_j)g_k^f,
\end{equation}
where $Q_j^\top Q_j$ projects onto the subspace spanned by the leader's buffered gradients, and $\alpha\in[0,1]$ controls the projection intensity. 
% is applied to the follower expert’s expert output  gradient $g_j^f$ individually, where $\alpha\in[0,1]$ is a projection intensity hyperparameter. Through this projection, the component of $g^f_j$ lying in the subspace spanned by the leader's buffered gradients is removed, allowing the follower to learn representations that diverge from the leader.
The projected gradient $g_k^{f*}$ is then used to update the follower $E_k^f$, while the leader $E_j^l$ retains its original task gradient $g_j^l$. This reduces redundant update directions and encourages the two experts to evolve toward differentiated output representations.
% After projection, the modified follower gradient $g_j^{f*}$ is used for the follower's subsequent parameter update. This projection allows the follower to learn representations that diverge from the leader.
% , which induces the experts output orthogonal.

\begin{figure*}[t]
    \centering

    \begin{subfigure}{0.495\linewidth}
        \centering
        \includegraphics[width=1\linewidth]{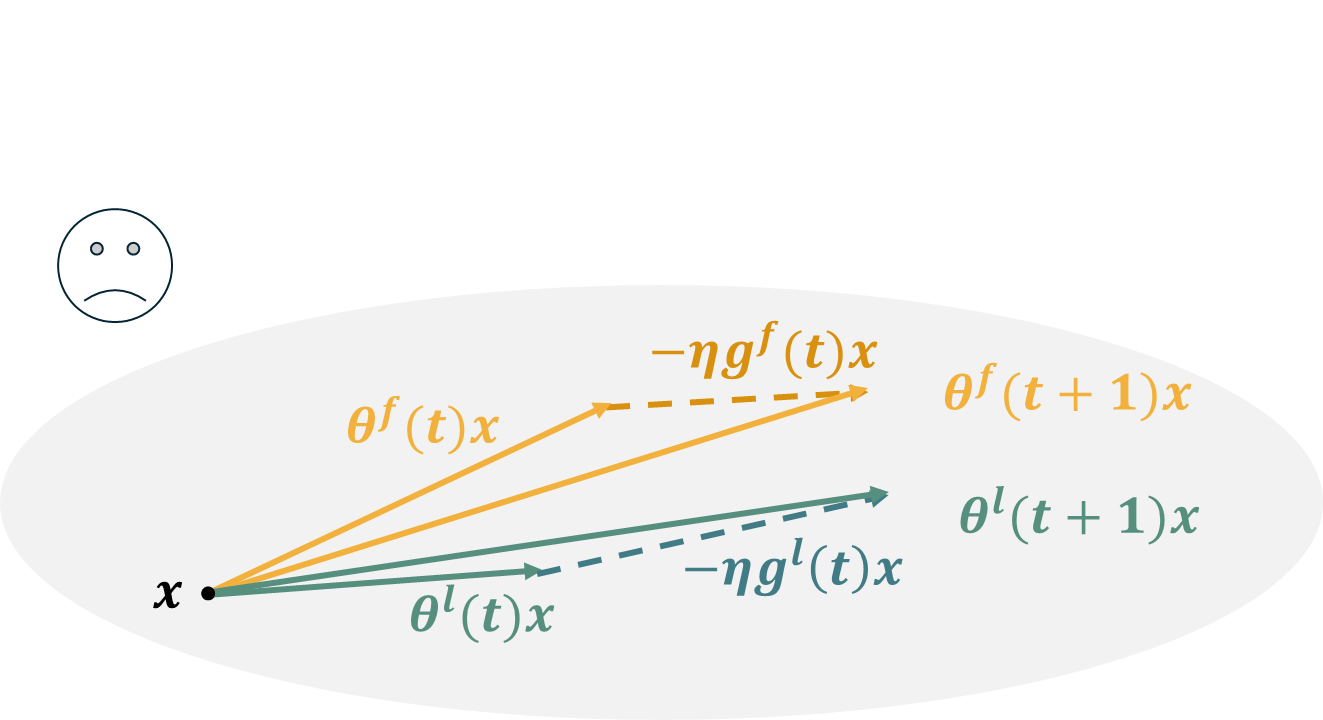}
        \caption{Insufficient Expert Specialization Without GaP-MoE}
        \label{fig:bad}
    \end{subfigure}
    \begin{subfigure}{0.495\linewidth}
        \centering
        \includegraphics[width=1\linewidth]{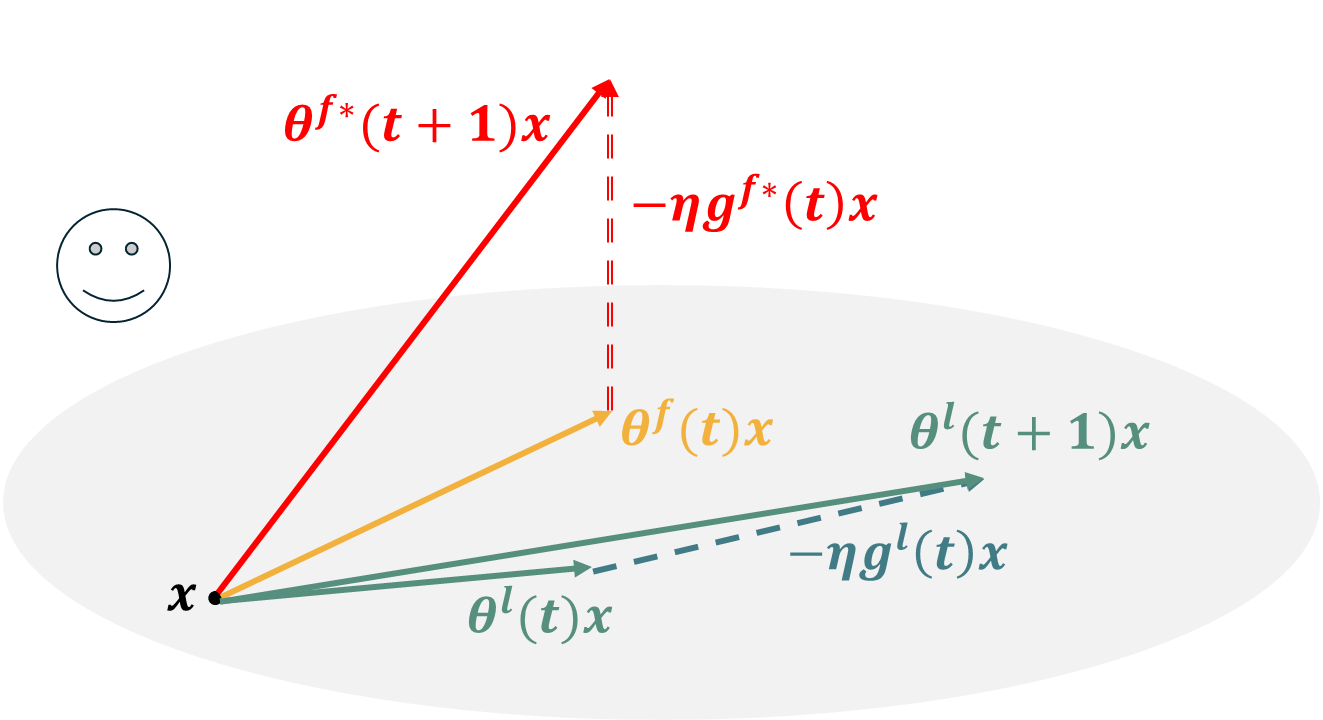}
        \caption{Explicit Specialization With GaP-MoE}
        \label{fig:good}
    \end{subfigure}
    \caption{Effect of gradient projection in the output representation space. The gray oval represents the past output gradient subspace of the leader expert ($g^l(\cdot)x$). The yellow and green lines represent the output vectors of the follower and leader experts, respectively, for the input token $x$, while the dotted lines represent their corresponding gradient update terms ($-\eta g(t)x$). If these output vectors are aligned with the same subspace, the output gradient term $-\eta g(t)x$ is also aligned with that subspace. However, by projecting the follower's gradient term, its output can escape from the existing subspace and encode different information, as illustrated by the red line.}
    \label{fig:fig2}
\end{figure*}

\subsection{Gradient Accumulation}\label{sec:accumulation}
GaP-MoE maintains a gradient buffer $Q_i \in \mathbb{R}^{h \times H}$ for each expert $i$, where $h$ is the buffer size. %Each buffer stores only the unprojected output gradients obtained when the corresponding expert is selected as a leader.

Because multiple tokens in a mini-batch may select the same expert as their leader, their output gradients are aggregated before being stored.
% Within a mini-batch, multiple tokens may be routed to the same expert. Therefore, instead of storing every output gradient individually, all leader output gradients accumulated for the corresponding expert within the current update are first averaged into a single gradient vector.
% Let $\mathcal{T}_i^{l}$ denote the set of tokens for which expert $i$ is selected as the leader within the current update:
Define
\begin{equation}
    \mathcal{T}_i^{l} = \{\, x : \text{Expert } E_i \text{ is selected as the leader for token } x \}.
\end{equation}
The corresponding mean gradient is then given by %averaged leader gradient is then given by
\begin{equation}
    \bar{g}_i^l = \frac{1}{|\mathcal{T}_i^{l}|} \sum_{\tau \in \mathcal{T}_i^{l}} g_{i,\tau}^l.
\end{equation}
Here, we write $g_{i,\tau}^l$ to denote the leader output gradient of expert $E_i$ for the $\tau$-th token in the current mini-batch.
To retain only a new gradient direction, $\bar{g}_i$ is orthogonalized against the current orthonormal basis vectors $b_i^{(1)},...,b_i^{(r)}$ in $Q_i$ as
\begin{equation}
    \hat{g}_i = \frac{\bar{g}_i}{\lVert \bar{g}_i \rVert_2},
    \quad v_i = \hat{g}_i - \sum_{m=1}^{r} \left( b_i^{(m)\top} \hat{g}_i \right) b_i^{(m)}.
\end{equation}
$v_i$ is normalized to obtain a new basis vector
\begin{equation}
    b_i^{\text{new}} = \frac{v_i}{\lVert v_i \rVert_2},
\end{equation}
which is appended to $Q_i$.
When the buffer reaches its maximum size $h$, the oldest vector is removed following a First-In-First-Out policy as
\begin{equation}
    Q_i \leftarrow
    \begin{cases}
        Q_i \cup \{b_i^{\text{new}}\}, & r < h \\[4pt]
        \left( Q_i \setminus \{b_i^{(1)}\} \right) \cup \{b_i^{\text{new}}\}, & r = h
    \end{cases}
\end{equation}
Consequently, $Q_i$ represents the recent unprojected update subspace of expert $E_i$ and is used for the projection in Section \ref{sec:projection}.

% The averaged gradient $\bar{g}_i$ is subsequently converted into an orthonormal basis vector via the Gram--Schmidt process before being stored in the buffer. Let $b_i^{(1)}, \dots, b_i^{(r)}$ denote the $r$ orthonormal basis vectors currently stored in $Q_i$, where $r \le h$. The new candidate vector is first normalized,
% \begin{equation}
%     \hat{g}_i = \frac{\bar{g}_i}{\lVert \bar{g}_i \rVert_2},v_i = \hat{g}_i - \sum_{m=1}^{r} \left( b_i^{(m)\top} \hat{g}_i \right) b_i^{(m)}
% \end{equation}
% then projected orthogonal to the existing basis,
% \begin{equation}
%     v_i = \hat{g}_i - \sum_{m=1}^{r} \left( b_i^{(m)\top} \hat{g}_i \right) b_i^{(m)},
% \end{equation}
% and re-normalized to obtain the new basis vector,
% \begin{equation}
%     b_i^{\text{new}} = \frac{v_i}{\lVert v_i \rVert_2}.
% \end{equation}

% The buffer $Q_i$ operates in a first-in, first-out (FIFO) manner: once the number of stored basis vectors reaches $h$, the oldest basis vector $b_i^{(1)}$ is discarded to accommodate the newest one.

% Through this process, the buffer $Q_i$ for expert $E_i$ comes to contain directional information about the gradients used in the leader expert update. The updated $Q_i$ is then used in the gradient projection process described in Section 3.2; the leader gradient generated anew during this process is averaged and used to update $Q_i$ in the next step.

\paragraph{Role-dependent Buffer Update}
Within a mini-batch, the same expert pair $(E_j, E_k)$ may take opposite roles for different tokens, appearing as $(E_j^l, E_k^f)$ for one token and $(E_j^f, E_k^l)$ for another token.
In this case, only the gradient received as the leader $g_j^l$ is used to update $Q_j$, while the projected gradient $g_j^{f*}$ is not used. This is to prevent the direction altered by the projection from being immediately reflected in the buffer, which would make the training process noisy.

\paragraph{Extension to $K>2$}
For $K > 2$, GaP-MoE selects one leader from the $TopK$ experts and treats the remaining $K-1$ experts as followers. The leader buffer is then used to project each follower gradient independently.
% the same dynamic leader selection method is used to designate a single leader and multiple followers, employing a structure in which the leader's buffer is projected onto each follower.

\subsection{Theoretical Justification of GaP-MoE}

\newtheorem{definition}{Definition}
\newtheorem{lemma}{Lemma}
\newtheorem{theorem}{Theorem}
\newtheorem{proposition}{Proposition}
\newtheorem{assumption}{Assumption}

GaP-MoE aims to prevent co-routed experts from being repeatedly updated in similar directions. As illustrated in Figure \ref{fig:bad}, when two experts receive highly aligned gradients for the same token, their output representations may continue to evolve toward similar functions.
In contrast, Figure \ref{fig:good} shows how GaP-MoE projects the follower gradient away from the leader's recent gradient subspace, thereby directing the follower toward a different update direction. GaP-MoE therefore reduces the gradient components shared by co-routed experts rather than directly constraining their output representations.

For analysis, let the outputs of the leader and follower experts for an input token $x$ be $y^l=\theta^l x$ and $y^f=\theta^f x$, respectively.
We measure their output difference as follows.
% The goal of GaP-MoE is to maximize the benefits of diversification by guiding the vectors $y_l = \theta_l x$ and $y_f = \theta_f x$—which are transformed from the same token using their respective parameters—to explore as diverse an output subspace as possible, thereby maximizing the benefits of diversification. 

\begin{definition}[Expert Output Difference]
\leavevmode\par
For an input token $x$, the expert output difference between the leader and follower experts is

\[
S(\theta^l,\theta^f;x)
=
\left\|y^l-y^f\right\|
=
\left\|(\theta^l-\theta^f)x\right\|.
\]
A larger $S$ indicates that the two experts produce more differentiated representations for the same input.
\end{definition}

To describe how gradient projection affects the outputs, consider the standard SGD update
% For theoretical discussion, we will multiply the general Stochastic Gradient Descent update equation,
\begin{equation}
\theta(t+1) = \theta(t) - \eta g_\theta(t),\label{eq:sgd}
\end{equation}
where $\eta$ is the learning rate and $t$ denotes the optimization step. Since $g_{\theta}(t)=\frac{\partial\mathcal{L}}{\partial\theta}=\frac{\partial\mathcal{L}}{\partial y}\frac{\partial y}{\partial \theta}=g(t)x^\top$, multiplying Eq. \eqref{eq:sgd} by $x$ gives the corresponding update in the expert-output space:
% by the input token $x$ to transform it into an equation for the output space: 
\begin{equation}
\theta(t+1)x = \theta(t)x - \eta g(t)||x||_2^2.\label{eq:sgd_update}
\end{equation}
% Here, $\eta$ is the learning rate, $t$ is the update time step, and $g_\theta = \frac{\partial L}{\partial \theta} = \frac{\partial L}{\partial y} \cdot \frac{\partial y}{\partial \theta} = \frac{\partial L}{\partial y} \cdot x^{\top} = g x^\top$ is the gradient passed to the expert parameter $\theta$. 

The output difference can be related to the angle between the expert outputs through the law of cosines:
\begin{align}
    S^2(\theta^l,\theta^f;x)=||\theta^l x||_2^2+||\theta^f x||_2^2-2||\theta^l x||_2||\theta^f x||_2 \cos\psi_x,
\end{align}
where $\cos\psi_x=\frac{(\theta^l x)^\top (\theta^f x)}{||\theta^l x||_2 ||\theta^f x||_2}$.
% Thus, when the two output norms are comparable, i.e., $\left\|\theta^l x\right\| \approx \left\|\theta^f x\right\| \approx c$, decreasing $\cos \psi_x$, or equivalently increasing the angle $\psi_x$, corresponds to increasing the output difference. This relation motivates GaP-MoE to suppress aligned expert-update directions that can repeatedly drive co-routed experts toward similar representation.
Thus, when  $\left\|\theta^l x\right\| \approx \left\|\theta^f x\right\| \approx c$, a larger angle $\psi_x$ corresponds to a larger output difference, motivating GaP-MoE to suppress aligned expert-update directions.
% Our goal is to maximize $S(\theta_l,\theta_f;x)$, and this can be expanded following the law of cosines.
% For a readable expansion, we assume that the output norms of two different experts are similar,
% $\left\|\theta_l x\right\| \approx \left\|\theta_f x\right\| \approx c$ this results in minimizing the following term:
% \begin{equation}
% 2c^2 (1 - \mathrm{cos}\phi_x)
% \end{equation}
% where $\mathrm{cos}\phi_x = \frac{(\theta_l x)^\top (\theta_f x)}{\left\|\theta_l x\right\| \left\|\theta_f x\right\|} \in [-1, 1]$.

% In other words, since this leads to the conclusion that we can reduce similarity by increasing the angle between outputs, the goal is to minimize $⟨\theta_l (t+1)x, \theta_f (t+1)x⟩$.
However, directly controlling $\psi_x$ during optimization would require an explicit output-level constraint. GaP-MoE instead promotes output differentiation by modifying the follower update direction. To characterize this mechanism, let $Q_j$ contain the orthonormal basis vectors accumulated from the recent leader gradients, and define $P_j=Q_j^\top Q_j$ as the orthogonal projector onto the corresponding gradient subspace.
Under full projection, i.e., $\alpha=1$, the follower output gradient becomes $g_k^{f*} = (I - P_j) g_k^f$. Using Eq. \eqref{eq:sgd_update}, the resulting follower output after one update becomes
\begin{align}
    \theta_k^{f*} (t+1)x = \theta_k^f (t)x - \eta(I - P_j) g_k^f \left\|x\right\|_2^2.\label{eq:follower_update}
\end{align}
Substituting Eq. \eqref{eq:follower_update} into the inner product with the leader, which is our target to reduce, yields
\begin{equation}
\begin{gathered}
\langle \theta_j^l(t)x - \eta g_j^l(t)x,\,
\theta_k^f(t)x - \eta g_k^{f*}(t)x \rangle \\
= \langle \theta_j^l(t)x, \theta_k^f(t)x \rangle
+ \langle -\eta g_j^l(t)x, \theta_k^f(t)x \rangle
+ \langle -\eta g_k^{f*}(t)x, \theta_j^l(t)x \rangle
+ \eta^2 \langle g_j^l(t)x, g_k^{f*}(t)x \rangle.
\end{gathered}
\label{eq3}
\end{equation}
% Let us denote the operation of projecting the follower’s gradient onto the leader’s buffer as $g_f^* = (I - Q_l^T Q_l) g_f = (I - P_l) g_f$. We assume that the direction of the expert output gradient, which consists $Q_i$ does not change significantly in the short term. 
% Applying this to the gradient update equation yields, we define the expert output after the projection operation as: 
% \begin{equation}
% \theta_f^* (t+1)x = \theta_f (t)x - \eta(I - P_l) g_f \left\|x\right\|^2
% \end{equation}
% Substituting this into the inner product expression with the leader expert output, 
% $⟨\theta_l (t+1)x, \theta_f^* (t+1)x⟩$, yields the equation below.
% \begin{align}
% &\langle\theta_l(t)x - \eta g_l(t)x,\ \theta_f(t)x - \eta g_f^*(t)x \rangle \notag\\
% &= \langle \theta_l(t)x, \theta_f(t)x \rangle \notag  + \langle -\eta g_l(t)x, \theta_f(t)x \rangle \notag\\
% &+ \langle -\eta g_f^*(t)x, \theta_l(t)x
% +\eta^2 \langle g_l(t)x, g_f^*(t)x \rangle\label{eq3}
% \end{align}
This equation builds the theoretical advantages of GaP-MoE in the following proposition. 
\begin{proposition}[Gradient Projection for Expert Separation]
The projection to follower output gradients can reduce $⟨\theta_j^l (t+1)x, \theta_k ^f(t+1)x⟩$ subject to the efficiency of MoE training. \label{prop1}
\end{proposition}
\begin{proof}
See Appendix~\ref{appendix:proof}.
% For \eqref{eq4},
% \begin{align}
% ⟨\frac{\partial L}{\partial y_l} \left\|x\right\|^2,(I-P)\frac{\partial L}{\partial y_f} \left\|x\right\|^2⟩\notag\\=⟨\frac{\partial L}{\partial y_l}x,(I-P)\frac{\partial L}{\partial y_f}x⟩ \left\|x\right\|^2\approx0\notag
% \end{align}
% This is because the projection $I-P$ made the vector $\frac{\partial L}{\partial y_f}x$ orthogonal to $\frac{\partial L}{\partial y_l}x$ by eliminating the overlapping direction of it.
% The parameters at time $t$ can be expressed by accumulating updates from an initial time $t_0$, where $t_0$ denotes the earliest time at which the short-term gradient consistency assumption holds:
% \begin{align}
% \theta(t)=\theta(t_0)-\eta\sum_{s=t_0}^{t}g(s)\\
% \text{where, } \langle g(s-1),g(s)\rangle\approx 0, \forall{s<t} \notag
% \end{align}
% Under this assumption, the accumulated gradients satisfy
% \begin{align}
% \sum_{s=t_0}^{t} g_l(s)
% \approx
% \alpha_l(t)g_l(t),
% \end{align}
% where $\alpha_l(t)$ is a scalar representing the accumulated update magnitude. Therefore, the parameter at time $t$ can be expressed as
% \begin{align}
% \theta(t)=\theta(t_0)-\alpha_l(t)g_l(t),
% \end{align}
% where $g_l(t)$ denotes the current unit gradient of expert $l$.
% Then,
% \begin{align}
% &\langle-\eta g_f^*(t)x,\theta_l(t)x\rangle\notag\\
% =&
% \langle-\eta g_f^*(t)x,\theta_l(t_0)x\rangle
% +\eta\alpha_l(t)\langle g_f^*(t)x,\theta_l(t)x\rangle\notag\\
% \approx&
% \langle-\eta g_f^*(t)x,\theta_l(t_0)x\rangle
% +0.
% \end{align}
% where the last approximation follows from the same argument as above.
\end{proof}
Orthogonalizing the output gradients serves to reduce the alignment between the two different expert output vectors through which a single token passes. 

It should be noted that we do not perform direct optimization for our reduction goal. Aggressively minimizing this similarity could drive the two representations toward opposite directions, i.e., $\psi=\pi$. In such a case, their contributions may cancel each other in the weighted sum of Eq. \eqref{eq_moe}, potentially hindering MoE optimization.
Instead, GaP-MoE orthogonalizes the follower update with respect to the leader gradient subspace, thereby alleviating excessive output alignment without explicitly forcing the expert outputs apart.
Proposition \ref{prop1} characterizes this mechanism by showing that the projection removes the follower-gradient component shared with the leader subspace, while minimally altering the original task gradient.
% Such optimization yields two output representation head toward opposite directions, i.e., $\psi_x$=$\pi$. However, if that were to happen, the influence of each output vector would be canceled out during the linear weight sum process (Eq    . \eqref{eq_moe}), which would hinder MoE training. Therefore, we perform orthogonalization to prevent this situation, and the role of Proposition~\ref{prop1} is to show that orthogonalization reduces the optimization goal.

% However, when the two vectors point in completely opposite directions, this becomes a major obstacle to the training of MoE, which uses a linear combination of outputs from multiple experts as its final output.

% GaP-MoE uses orthogonal projection to increase the angle between gradients without significantly hindering the learning process. At the same time, dynamic leader selection fully protects the gradient of a single expert at each step, thereby directly ensuring expert specialization. This suggests that even if the MoE architectures described above do not perform as intended, GaP-MoE can serve as a minimal safeguard during training.

\section{Experiments}

Although MoE is being applied in various fields thanks to its scalability  \citep{VMoE,you2021speechmoe,lin2026moe}, this paper focuses on natural language tasks. 

Since GaP-MoE is designed to deliver relative performance improvements when combined with other models, we compare it to the baseline SMoE \citep{shazeer2017} across various tasks and subsequently compare the performance changes when combined with various models already published for specific tasks. For all tasks, each feed-forward sub-layer of the baseline models is replaced with a MoE layer consisting of four experts—composed of two linear layers and a GELU—and the $Top2$ routing method is employed.
For more detailed experimental setup, see Appendix~\ref{appedinx:hyperparam}.

The hyperparameter controlling the projection intensity is set to its maximum value of 1 for all tasks; however, during the first epoch, it is set to increase linearly from 0.1 to 1.0 in steps to mitigate the impact of noisy gradient changes. Additionally, the projection threshold $\gamma$ for GaP-MoE is set to -1 and 0.7, respectively, to represent both full projection and minimal projection scenarios.

We evaluate GaP-MoE in both fine-tuning and from-scratch scenarios.

\subsection{Fine-tuning}
Based on BERT \citep{devlin2019bert}, we use the publicly available pretrained version, bert-base-uncased, obtained via the HuggingFace Transformers library~\citep{wolf2020transformers}. 

We replace the 12 feed-forward layers with MoE layers consisting of four experts of the same parameter size, and initialize the experts identically based on the pretrained parameters. For the tasks, we use the GLUE benchmark \citep{wang2018glue}, which includes a total of eight NLP tasks such as sentence reasoning and sentiment classification, and evaluate the model using task-specific metrics such as accuracy and F1 score. After training for five epochs, we summarize the final results based on the best performance. 

Table~\ref{fine} shows the fine-tuning evaluation results.
In the fine-tuning experiments reported here, GaP-MoE is applied to all 12 MoE layers. While this configuration yields consistent but modest improvements in average performance over the SMoE baseline, we are investigating whether selectively applying GaP-MoE to a subset of layers can provide further benefits.

It is worth noting that performing projection at all times yielded higher results on average than projecting only when the last basis similarity was high. This suggests that, in MoE architectures, gradient projection more often aids expert specialization than hinders learning.

\begin{table*}[t]
\centering

\begin{tabular}{l|cccccccc|c}
  \toprule
  Models & CoLA & SST-2 & MRPC & STS-B & QQP & MNLI &QNLI & RTE & \textbf{Avg.}\\
  \hline
  SMoE & 56.00 & 93.16 & 85.21 & \textbf{89.28} & 91.20 & 84.54 & 91.45 & 64.86 & 81.96\\
  GaP-MoE($\gamma=-1$) & 56.37 & \textbf{93.20} & 84.64 & 89.22 & \textbf{91.28} & \textbf{84.75} & \textbf{91.58} & \textbf{65.46} & \textbf{82.06} \\
  GaP-MoE($\gamma=0.7$) & \textbf{56.44} & 93.00 & \textbf{85.29} & 89.27 & 91.28 & 84.52 & 91.45 & 64.62 & 81.98 \\
  \bottomrule
\end{tabular}
\caption{Evaluation results on the GLUE benchmark. For evaluation metrics, CoLA uses the Matthews correlation coefficient, STS-B uses Pearson’s correlation coefficient, and all others use accuracy. For all of them, the higher the score, the better. Results are averaged over three seeds.}
\label{fine}
\end{table*}

\begin{table}[t]
\centering

\begin{tabular}{l|ccc}
  \toprule
  Models & WikiText-103 & enwik8 & text8\\
  \hline
  SMoE & 78.80 & 2.005& 1.962\\
  GaP-MoE($\gamma=0.7$) & \textbf{76.50} & \textbf{1.953} & \textbf{1.945}\\
  \bottomrule
\end{tabular}
\caption{Evaluation results of from-scratch environments. For evaluation metrics, WikiText-103 uses perplexity, while enwik8 and text8 use BPC. For all of them, the lower the score, the better. Results are averaged over three seeds.}
\label{from}
\end{table}

\begin{table}[t]
\centering
\begin{minipage}[t]{0.36\textwidth}
\vspace{0pt}
\small

\begin{minipage}[t][3.8cm][c]{\linewidth}
\vspace{0pt}
\centering
\begin{tabular}{l|c}
  \toprule
  Models & WikiText-103 \\
  \midrule
  SMoE & 83.63 \\
  GaP-MoE($\gamma=-1$) & \textbf{81.35} \\
  GaP-MoE($\gamma=0.7$) & 82.44 \\
  \bottomrule
\end{tabular}
\end{minipage}

\caption{Performance on the WikiText-103 dataset without router training. Instead of using a trained router, two experts were randomly selected for each token. The lower the score, the better. Results are averaged over three seeds.}
\label{uni}
\end{minipage}
\hfill
\begin{minipage}[t]{0.61\textwidth}
\vspace{0pt}
\small

\begin{minipage}[t][3.8cm][c]{\linewidth}
\vspace{0pt}
\centering
\setlength{\tabcolsep}{4pt}
\begin{tabular}{l|ccc}
  \toprule
  Models & CoLA & STS-B & RTE \\
  \midrule
  XMoE & 58.65 & 86.65 & 64.86 \\
  XMoE+GaP-MoE($\gamma=-1$)
    & 57.80 & \textbf{86.67} & 65.10 \\
  XMoE+GaP-MoE($\gamma=0.7$)
    & \textbf{58.76} & 86.36 & \textbf{65.82} \\
  \midrule
  SimSMoE & 57.89 & 88.90 & 65.22 \\
  SimSMoE+GaP-MoE($\gamma=-1$)
    & 58.01 & \textbf{88.97} & 65.46 \\
  SimSMoE+GaP-MoE($\gamma=0.7$)
    & \textbf{58.09} & 88.95 & \textbf{65.58} \\
  \bottomrule
\end{tabular}
\end{minipage}

\caption{Evaluation results of generalizability experiments. For evaluation metrics, CoLA uses the Matthews correlation coefficient, STS-B uses Pearson's correlation coefficient, and RTE uses accuracy. For all of them, the higher the score, the better. Results are averaged over three seeds.}
\label{gentest}
\end{minipage}
\end{table}

\begin{figure}[t]
\centering

\includegraphics[width=0.4\columnwidth]{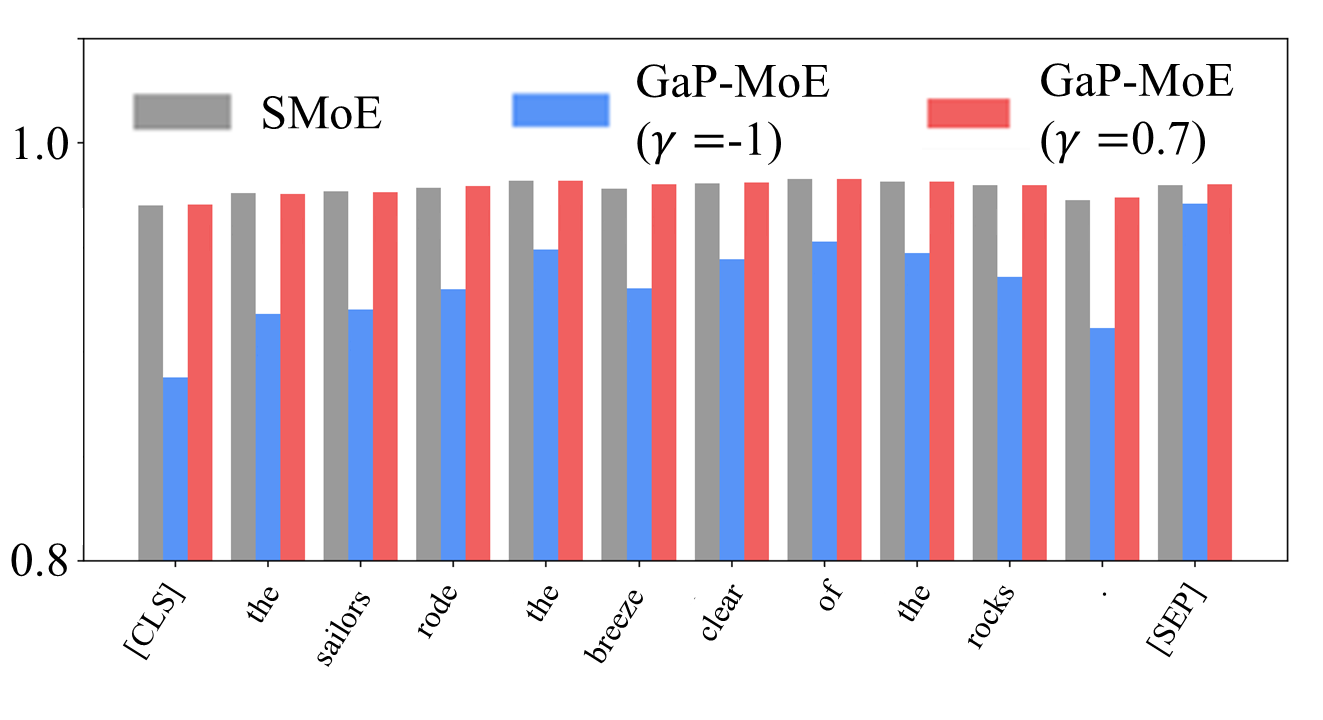} 
\caption{Each bar represents, in order, the expert output similarity of SMoE and GaP-MoE with $\gamma = -1$ and $\gamma = 0.7$, respectively, trained using a single seed on COLA. The x-axis shows each token passed through the layers in the first sequence of COLA’s evaluation set.}
\label{fig3}

\vspace{12pt}

\centering

\begin{subfigure}{0.3\linewidth}
    \centering
    \includegraphics[width=\linewidth]{Figures/Angle_SMoE.PNG}
    \caption{SMoE}
    \label{fig:first}
\end{subfigure}
\begin{subfigure}{0.3\linewidth}
        \centering
        \includegraphics[width=\linewidth]{Figures/Angle_-1.PNG}
        \caption{GaP-MoE($\gamma$=-1)}
        \label{fig:second}
\end{subfigure}
\begin{subfigure}{0.3\linewidth}
        \centering
        \includegraphics[width=\linewidth]{Figures/Angle_0.7.PNG}
        \caption{GaP-MoE($\gamma$=0.7)}
        \label{fig:third}
\end{subfigure}
\caption{The angle between expert output vectors in the sixth layer when a single token is input into the three models that have been trained on the QNLI dataset. $\angle$HD is the angle of the output vectors calculated in the original 768-dimensional output space. PCA was applied to each vector to visualize them in two dimensions, and the angle in this two-dimensional space is denoted as $\angle$2D. The red lines represent the outputs of the experts selected as $Top2$, while the black dashed lines show the results of passing the token through experts not selected as $Top2$. Results for all layers of the token are provided in Appendix~\ref{appedinx:angle}. The number in parentheses refers to the projection threshold $\gamma$.}
\label{fig:comparison}
\end{figure}

\subsection{From-scratch}

We use GPT-2 \citep{radford2019language} as the backbone model. To fit our computational budget, we adopt a reduced GPT-2 Small configuration with 6 transformer layers, 8 attention heads, and an intermediate dimension of 2048. The original GPT-2 Small consists of 12 layers, 12 attention heads, and an intermediate dimension of 3072. Additionally, for the from-scratch task, we replaced only the last two layers with GaP-MoE—which exhibit relatively stronger gradient consistency, as will be discussed in \ref{consistency}—while retaining the standard SMoE layers for the rest of the network.

%As in the BERT experiments, each feed-forward layer is replaced with an MoE layer, while 
The expert parameters are randomly initialized. We conduct experiments on WikiText-103 \citep{merity2016pointer}, enwik8, and text8 \citep{mahoney2011large}, all of which are evaluated on the next-token prediction task. WikiText-103 is evaluated using perplexity (PPL), whereas enwik8 and text8 are evaluated using bits per character (BPC). The validation loss is evaluated every 1,000 training steps, and training is terminated if the validation loss does not improve for three consecutive evaluations. 

Table~\ref{from} presents the from-scratch evaluation results. Compared to the baseline SMoE, from-scratch experiments also showed that GaP-MoE generally performed better. However, on text8—a relatively easy dataset created by retaining only alphabetic characters and spaces from enwik8—the performance improvement wasn't as significant as it was for the other two tasks.

\subsection{Expert Specialization Without Router Training}
GaP-MoE is particularly robust to router failure because it induces expert specialization directly through the gradients generated during training. Table \ref{uni} shows the difference in model performance on the WikiText-103 dataset when the router was not trained at all, and instead two experts were randomly selected and trained. Although performance decreased significantly overall compared to Table \ref{from}, the performance drop is considerably smaller for GaP-MoE. In other words, this suggests that GaP-MoE explicitly guides the process to prevent expert specialization from failing.

\subsection{Generalizability with Different MoE}

One of the novelties of GaP-MoE is that it can be easily applied to various MoE architectures. We conducted additional experiments by applying two models—XMoE \citep{chi2022representation} and SimSMoE \citep{do2025simsmoe}, which address the same problem to our GaP-MoE experimental framework. 

Since no official implementation of XMoE is publicly available, and the official SimSMoE implementation is designed for pre-training rather than our fine-tuning setting, we implemented the core mechanisms of both methods within our unified experimental framework.

To efficiently evaluate the various experimental results, we categorized the tasks in the GLUE benchmark into three groups and conducted one experiment for each task. These are CoLA for the single-sentence classification task, STS-B for the sentence pair similarity task, and RTE for the natural language inference task.

Table \ref{gentest} shows the results of evaluating each model by applying the GaP-MoE architecture to three fine-tuning tasks. 
Interestingly, while applying GaP-MoE resulted in improved performance on most tasks, it was found to be more significantly affected by $\gamma$ than the GaP-MoE-only model. Identifying the appropriate $\gamma$ value for each task and model should be addressed in future research.

\subsection{Expert Output Differentiation}

Figure \ref{fig3} shows the cosine similarity between expert outputs in the final layer for all tokens in an input sequence, whereas Figure \ref{fig:comparison} visualizes the angles between expert outputs in the sixth layer using a single input token.

When the projection threshold $\gamma$ was set to $-1$, the similarity decreased substantially. In contrast, setting $\gamma$ to $0.7$ resulted in only a marginal reduction in similarity, as projection was applied only when the similarity exceeded the threshold. This suggests that gradient projection can be used to differentiate expert outputs. 

%However, finding an appropriate value for $\gamma$ that strikes a balance between differentiation and learning may be difficult due to task-specific characteristics or the presence of other modules. When using GaP-MoE alone, we recommend applying a low value to allow it to directly influence the learning process; when using it in conjunction with other modules, we recommend using a high value to serve as a safeguard against representation collapse.

\subsection{Short-term Gradient Consistency}
\label{consistency}

To empirically examine whether expert gradients exhibit the short-term gradient consistency assumed in the theoretical justification, we track for each expert of the first and the last layer, the cosine similarity between its current output gradient and the orthonormal basis accumulated $t$ steps earlier ($t\in\{1,3,5,10\}$) whenever that expert is selected as leader. 

Figure \ref{fig:consistency} visualizes these trajectories across layers and gap sizes. In the first step, the cosine similarity values are generally greater than 0, but they fluctuate significantly. In contrast, when the model reaches the intermediate step, the cosine similarity stabilizes. To avoid noisy updates, it is recommended to gradually increase the projection intensity $\alpha$ from a low level to a high level during the initial training phase.
Once the updates have stabilized to some extent, the variance in similarity decreases. At this point, the early layers maintain a similarity level above 0 but converge toward 0 as the gap widens. Conversely, the later layers consistently maintain a high level of similarity. In other words, since there are some differences in short-term consistency across layers, theoretical performance differences may exist between layers; therefore, we expect that performance will be enhanced by not unnecessarily increasing the buffer size. In our experiments, we used buffer sizes of 8 and 16 (for training from scratch and fine-tuning, respectively).

\begin{figure*}[t]
    \centering

    \begin{subfigure}{0.495\linewidth}
        \centering
        \includegraphics[width=\linewidth]{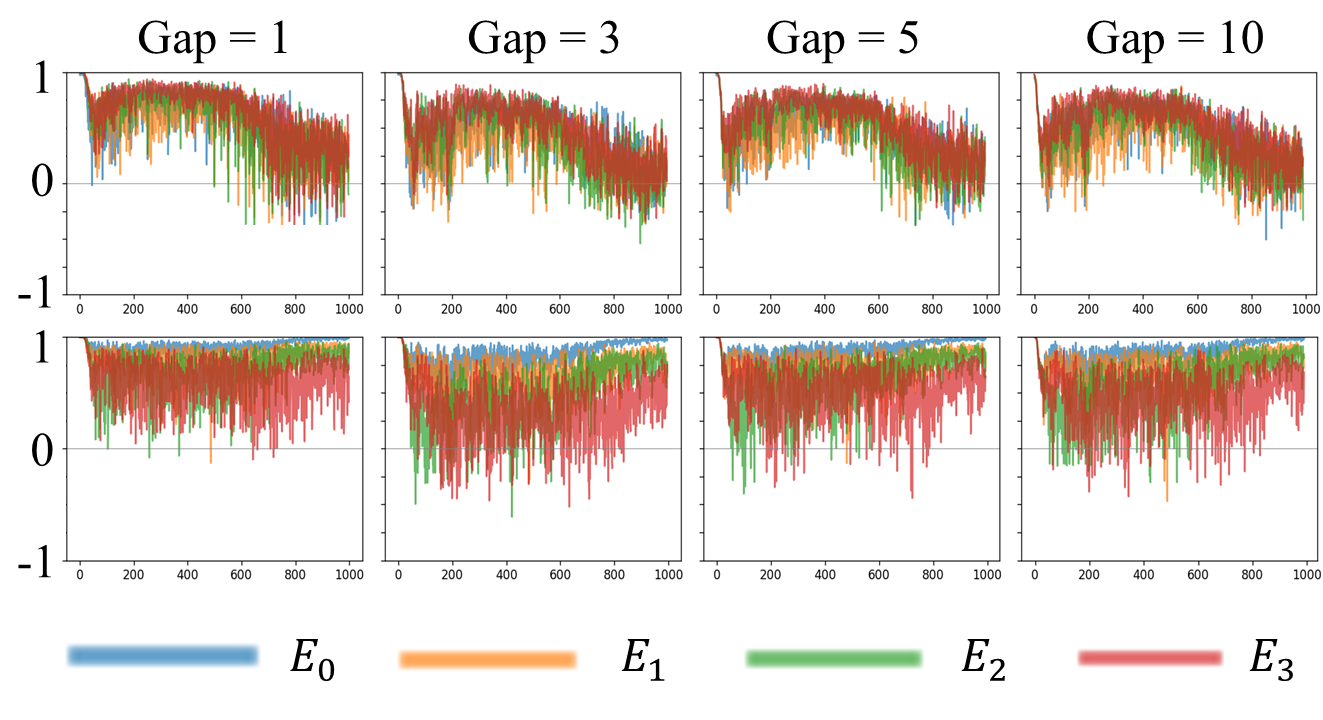}
        \caption{Similarity for each expert at the first step}
        \label{fig:dyn}
    \end{subfigure}
    \begin{subfigure}{0.495\linewidth}
        \centering
        \includegraphics[width=\linewidth]{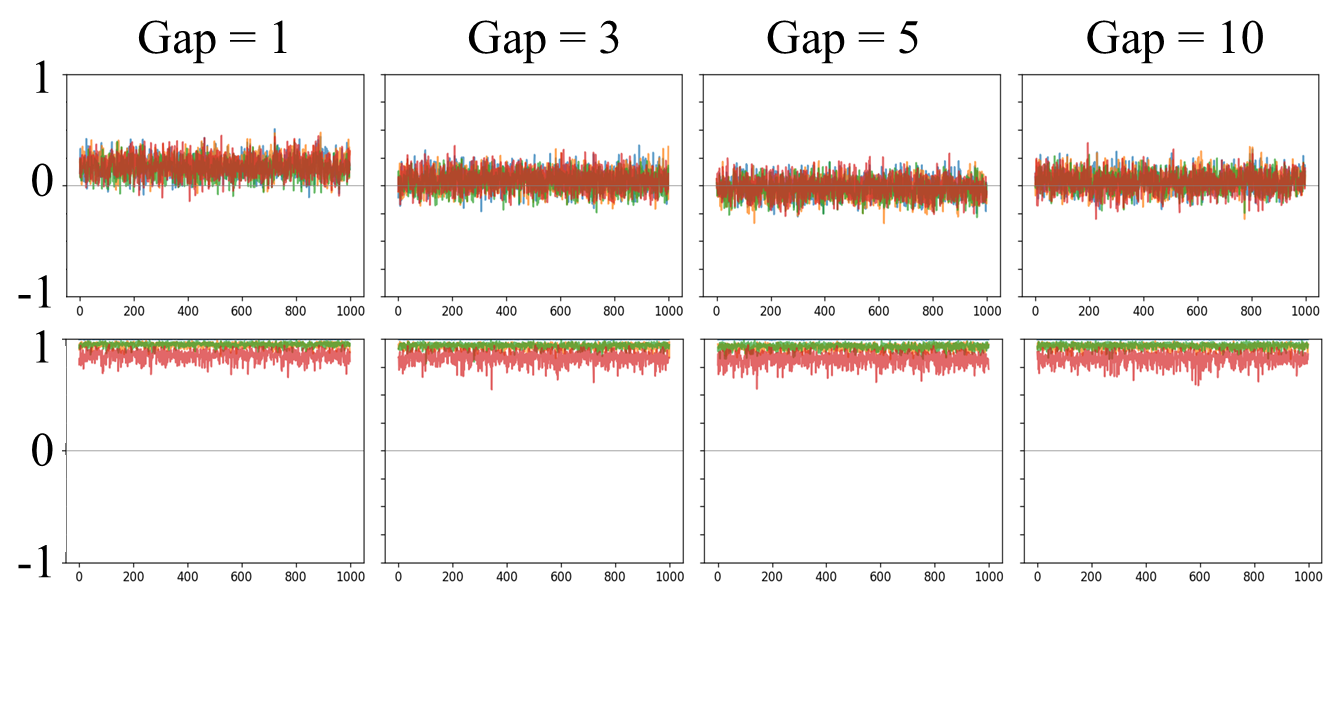}
        \caption{Similarity for each expert at the intermediate step}
        \label{fig:proj}
    \end{subfigure}
    \caption{In both cases, the top figure shows the first MoE layer, and the bottom figure shows the last MoE layer; the $t\in[1,3,5,10]$ used to set the Gap indicates how many time steps back the expert was compared to.}
    \label{fig:consistency}
\end{figure*}

\section{Conclusion and Future Work}

This study proposed GaP-MoE, which retains the standard training procedure of sparse MoE with $TopK$ routing while imposing explicit orthogonalization constraints only on pre-computed expert output gradients. Through dynamic leader selection and gradient projection, GaP-MoE encourages experts to learn more diverse output representations without introducing additional parameter constraints or requiring gradient recomputation. Experimental results demonstrated that GaP-MoE achieves competitive or improved performance over the baseline SMoE and can be effectively combined with existing specialization methods, demonstrating its compatibility with different MoE architectures.

One limitation of this work is that the hyperparameters introduced by GaP-MoE, including the buffer size $h$, the choice of layers to which GaP-MoE is applied during fine-tuning, projection threshold $\gamma$, and projection intensity $\alpha$, were not systematically explored. Consequently, the reported settings may not be optimal, and the relationship between these hyperparameters and model performance remains insufficiently understood. In addition, our evaluation was limited to relatively small-scale models. Future work will therefore focus on systematic hyperparameter optimization and validating GaP-MoE on larger-scale language models.

\bibliographystyle{unsrtnat}
\bibliography{references}  %%% Uncomment this line and comment out the ``thebibliography'' section below to use the external .bib file (using bibtex) .

@article{scaling01,
  title={Scaling Laws for Neural Language Models},
  author={Kaplan, Jared and others},
  journal={arXiv preprint arXiv:2001.08361},
  year={2020}
}

@article{shazeer2017,
  title={Outrageously large neural networks: The sparsely-gated mixture-of-experts layer},
  author={Shazeer, Noam and Mirhoseini, Azalia and Maziarz, Krzysztof and Davis, Andy and Le, Quoc and Hinton, Geoffrey and Dean, Jeff},
  journal={arXiv preprint arXiv:1701.06538},
  year={2017}
}

@article{wang2024auxiliary,
  title={Auxiliary-loss-free load balancing strategy for mixture-of-experts},
  author={Wang, Lean and Gao, Huazuo and Zhao, Chenggang and Sun, Xu and Dai, Damai},
  journal={arXiv preprint arXiv:2408.15664},
  year={2024}
}

@article{lepikhin2020gshard,
  title={Gshard: Scaling giant models with conditional computation and automatic sharding},
  author={Lepikhin, Dmitry and Lee, HyoukJoong and Xu, Yuanzhong and Chen, Dehao and Firat, Orhan and Huang, Yanping and Krikun, Maxim and Shazeer, Noam and Chen, Zhifeng},
  journal={arXiv preprint arXiv:2006.16668},
  year={2020}
}

@inproceedings{do2023hyperrouter,
  title={Hyperrouter: Towards efficient training and inference of sparse mixture of experts},
  author={Do, Truong Giang and Khiem, Le and Pham, Quang and Nguyen, Trungtin and Doan, Thanh-Nam and Nguyen, Binh and Liu, Chenghao and Ramasamy, Savitha and Li, Xiaoli and Hoi, Steven},
  booktitle={Proceedings of the 2023 conference on empirical methods in natural language processing},
  pages={5754--5765},
  year={2023}
}

@inproceedings{nguyen2025statistical,
  title={Statistical advantages of perturbing cosine router in mixture of experts},
  author={Nguyen, Huy and Akbarian Saravi, Pedram and Pham, Trang and Nguyen Vu, Thien Trang and Zhang, Shujian and Ho, Nhat},
  booktitle={International Conference on Learning Representations},
  volume={2025},
  pages={95812--95844},
  year={2025}
}

@article{huang2026sd,
  title={SD-MoE: Spectral Decomposition for Effective Expert Specialization},
  author={Huang, Ruijun and Dong, Fang and Zhang, Xin and Cao, Hengjie and Huang, Zhendong and Chen, Anrui and Zhou, Jixian and Chen, Mengyi and Yang, Yifeng and Dong, Mingzhi and others},
  journal={arXiv preprint arXiv:2602.12556},
  year={2026}
}

@article{fedus2022switch,
  title={Switch transformers: Scaling to trillion parameter models with simple and efficient sparsity},
  author={Fedus, William and Zoph, Barret and Shazeer, Noam},
  journal={Journal of Machine Learning Research},
  volume={23},
  number={120},
  pages={1--39},
  year={2022}
}

@inproceedings{zhang2022mixture,
  title={Mixture of attention heads: Selecting attention heads per token},
  author={Zhang, Xiaofeng and Shen, Yikang and Huang, Zeyu and Zhou, Jie and Rong, Wenge and Xiong, Zhang},
  booktitle={Proceedings of the 2022 Conference on Empirical Methods in Natural Language Processing},
  pages={4150--4162},
  year={2022}
}

@inproceedings{zhang2023robust,
  title={Robust mixture-of-expert training for convolutional neural networks},
  author={Zhang, Yihua and Cai, Ruisi and Chen, Tianlong and Zhang, Guanhua and Zhang, Huan and Chen, Pin-Yu and Chang, Shiyu and Wang, Zhangyang and Liu, Sijia},
  booktitle={Proceedings of the IEEE/CVF International Conference on Computer Vision},
  pages={90--101},
  year={2023}
}

@article{guo2026advancing,
  title={Advancing expert specialization for better moe},
  author={Guo, Hongcan and Lu, Haolang and Nan, Guoshun and Chu, Bolun and Zhuang, Jialin and Yang, Yuan and Che, Wenhao and Cao, Xinye and Leng, Sicong and Cui, Qimei and others},
  journal={Advances in Neural Information Processing Systems},
  volume={38},
  pages={48767--48809},
  year={2026}
}

@inproceedings{yang2025solving,
  title={Solving token gradient conflict in mixture-of-experts for large vision-language model},
  author={Yang, Longrong and Shen, Dong and Cai, Chaoxiang and Yang, Fan and Gao, Tingting and Zhang, Di and Li, Xi},
  booktitle={International Conference on Learning Representations},
  volume={2025},
  pages={33561--33587},
  year={2025}
}

@inproceedings{liu2024diversifying,
  title={Diversifying the mixture-of-experts representation for language models with orthogonal optimizer},
  author={Liu, Boan and Ding, Liang and Shen, Li and Peng, Keqin and Cao, Yu and Cheng, Dazhao and Tao, Dacheng},
  booktitle={ECAI 2024: 27th European Conference on Artificial Intelligence, 19--24 October 2024, Santiago de Compostela, Spain--Including 13th Conference on Prestigious Applications of Intelligent Systems (PAIS 2024)},
  pages={2966--2973},
  year={2024},
  organization={SAGE Publications Pvt. Ltd 1 Oliver's Yard, 55 City Road, London, EC1Y 1SP}
}

@article{chi2022representation,
  title={On the representation collapse of sparse mixture of experts},
  author={Chi, Zewen and Dong, Li and Huang, Shaohan and Dai, Damai and Ma, Shuming and Patra, Barun and Singhal, Saksham and Bajaj, Payal and Song, Xia and Mao, Xian-Ling and others},
  journal={Advances in Neural Information Processing Systems},
  volume={35},
  pages={34600--34613},
  year={2022}
}

@article{pham2024competesmoe,
  title={CompeteSMoE--Effective Training of Sparse Mixture of Experts via Competition},
  author={Pham, Quang and Do, Giang and Nguyen, Huy and Nguyen, TrungTin and Liu, Chenghao and Sartipi, Mina and Nguyen, Binh T and Ramasamy, Savitha and Li, Xiaoli and Hoi, Steven and others},
  journal={arXiv preprint arXiv:2402.02526},
  year={2024}
}

@inproceedings{do2025simsmoe,
  title={Simsmoe: Toward efficient training mixture of experts via solving representational collapse},
  author={Do, Giang and Le, Hung and Tran, Truyen},
  booktitle={Findings of the Association for Computational Linguistics: NAACL 2025},
  pages={2012--2025},
  year={2025}
}

@inproceedings{dai2024deepseekmoe,
  title={Deepseekmoe: Towards ultimate expert specialization in mixture-of-experts language models},
  author={Dai, Damai and Deng, Chengqi and Zhao, Chenggang and Xu, RX and Gao, Huazuo and Chen, Deli and Li, Jiashi and Zeng, Wangding and Yu, Xingkai and Wu, Yu and others},
  booktitle={Proceedings of the 62nd Annual Meeting of the Association for Computational Linguistics (Volume 1: Long Papers)},
  pages={1280--1297},
  year={2024}
}

@article{krishnamurthy2023improving,
  title={Improving expert specialization in mixture of experts},
  author={Krishnamurthy, Yamuna and Watkins, Chris and Gaertner, Thomas},
  journal={arXiv preprint arXiv:2302.14703},
  year={2023}
}

@article{MoE,
  title={Adaptive mixtures of local experts},
  author={Jacobs, Robert A and Jordan, Michael I and Nowlan, Steven J and Hinton, Geoffrey E},
  journal={Neural computation},
  volume={3},
  number={1},
  pages={79--87},
  year={1991},
  publisher={MIT Press}
}

@article{gpt3,
  title={Language Models are Few-Shot Learners},
  author={Brown, Tom B. and others},
  journal={Advances in Neural Information Processing Systems},
  year={2020}
}

@article{trans,
  title={Attention is all you need},
  author={Vaswani, Ashish and Shazeer, Noam and Parmar, Niki and Uszkoreit, Jakob and Jones, Llion and Gomez, Aidan N and Kaiser, {\L}ukasz and Polosukhin, Illia},
  journal={Advances in neural information processing systems},
  volume={30},
  year={2017}
}

@article{yu2020gradient,
  title={Gradient surgery for multi-task learning},
  author={Yu, Tianhe and Kumar, Saurabh and Gupta, Abhishek and Levine, Sergey and Hausman, Karol and Finn, Chelsea},
  journal={Advances in neural information processing systems},
  volume={33},
  pages={5824--5836},
  year={2020}
}

@article{lopez2017gradient,
  title={Gradient episodic memory for continual learning},
  author={Lopez-Paz, David and Ranzato, Marc'Aurelio},
  journal={Advances in neural information processing systems},
  volume={30},
  year={2017}
}

@inproceedings{farajtabar2020orthogonal,
  title={Orthogonal gradient descent for continual learning},
  author={Farajtabar, Mehrdad and Azizan, Navid and Mott, Alex and Li, Ang},
  booktitle={International conference on artificial intelligence and statistics},
  pages={3762--3773},
  year={2020},
  organization={PMLR}
}

@inproceedings{VMoE,
 author = {Riquelme, Carlos and Puigcerver, Joan and Mustafa, Basil and Neumann, Maxim and Jenatton, Rodolphe and Susano Pinto, Andr\'{e} and Keysers, Daniel and Houlsby, Neil},
 booktitle = {Advances in Neural Information Processing Systems},
 editor = {M. Ranzato and A. Beygelzimer and Y. Dauphin and P.S. Liang and J. Wortman Vaughan},
 pages = {8583--8595},
 publisher = {Curran Associates, Inc.},
 title = {Scaling Vision with Sparse Mixture of Experts},
 url = {https://proceedings.neurips.cc/paper_files/paper/2021/file/48237d9f2dea8c74c2a72126cf63d933-Paper.pdf},
 volume = {34},
 year = {2021}
}

@article{you2021speechmoe,
  title={Speechmoe: Scaling to large acoustic models with dynamic routing mixture of experts},
  author={You, Zhao and Feng, Shulin and Su, Dan and Yu, Dong},
  journal={arXiv preprint arXiv:2105.03036},
  year={2021}
}

@article{lin2026moe,
  title={Moe-llava: Mixture of experts for large vision-language models},
  author={Lin, Bin and Tang, Zhenyu and Ye, Yang and Huang, Jinfa and Zhang, Junwu and Pang, Yatian and Jin, Peng and Ning, Munan and Luo, Jiebo and Yuan, Li},
  journal={IEEE Transactions on Multimedia},
  year={2026},
  publisher={IEEE}
}

@inproceedings{devlin2019bert,
  title={Bert: Pre-training of deep bidirectional transformers for language understanding},
  author={Devlin, Jacob and Chang, Ming-Wei and Lee, Kenton and Toutanova, Kristina},
  booktitle={Proceedings of the 2019 conference of the North American chapter of the association for computational linguistics: human language technologies, volume 1 (long and short papers)},
  pages={4171--4186},
  year={2019}
}

@inproceedings{wolf2020transformers,
  title={Transformers: State-of-the-art natural language processing},
  author={Wolf, Thomas and Debut, Lysandre and Sanh, Victor and Chaumond, Julien and Delangue, Clement and Moi, Anthony and Cistac, Pierric and Rault, Tim and Louf, R{\'e}mi and Funtowicz, Morgan and others},
  booktitle={Proceedings of the 2020 conference on empirical methods in natural language processing: system demonstrations},
  pages={38--45},
  year={2020}
}

@inproceedings{wang2018glue,
  title={GLUE: A multi-task benchmark and analysis platform for natural language understanding},
  author={Wang, Alex and Singh, Amanpreet and Michael, Julian and Hill, Felix and Levy, Omer and Bowman, Samuel R},
  booktitle={Proceedings of the 2018 EMNLP workshop BlackboxNLP: Analyzing and interpreting neural networks for NLP},
  pages={353--355},
  year={2018}
}

@article{radford2019language,
  title={Language models are unsupervised multitask learners},
  author={Radford, Alec and Wu, Jeffrey and Child, Rewon and Luan, David and Amodei, Dario and Sutskever, Ilya and others},
  journal={OpenAI blog},
  volume={1},
  number={8},
  pages={9},
  year={2019}
}

@article{merity2016pointer,
  title={Pointer sentinel mixture models},
  author={Merity, Stephen and Xiong, Caiming and Bradbury, James and Socher, Richard},
  journal={arXiv preprint arXiv:1609.07843},
  year={2016}
}

@misc{mahoney2011large,
  title={Large text compression benchmark},
  author={Mahoney, Matt},
  year={2011},
  howpublished={\url{http://mattmahoney.net/dc/text.html}}
}

@article{wang2026myth,
  title={The myth of expert specialization in moes: Why routing reflects geometry, not necessarily domain expertise},
  author={Wang, Xi and Hayou, Soufiane and Nalisnick, Eric},
  journal={arXiv preprint arXiv:2604.09780},
  year={2026}
}

\clearpage
\appendix

\section*{Appendix}

\section{Algorithm of GaP-MoE}
\label{appedinx:full_algorithm}

\begin{algorithm}[H]
\caption{GaP-MoE Update (Top-2 Routing)}
\label{alg:algorithm}
\textbf{Input}: mini-batch of tokens $\{x_n\}_{n=1}^{B}$; routing scores $G(x_n)$ from router \\
\textbf{Parameter}: expert weights $\{\theta_i\}_{i=1}^{N}$; gradient buffers $\{Q_i\}_{i=1}^{N}$; buffer size $h$; projection threshold $\gamma$; projection intensity $\alpha$ \\
\textbf{Output}: updated expert weights $\{\theta_i(t+1)\}_{i=1}^{N}$, updated buffers $\{Q_i\}$
\begin{algorithmic}[1]
\STATE Initialize $\mathcal{T}_i^{l} \leftarrow \emptyset,\ \forall i \in [1, N]$
\FOR{each token $x_n$ in the mini-batch}
    \STATE Select top-2 experts $E_j, E_k$ via routing scores $G(x_n)$
    \STATE Normalize scores over the selected pair to obtain $p_j, p_k$
    \STATE Sample leader $l \sim \mathrm{Categorical}(p_j, p_k)$
    \IF{$l = j$}
        \STATE Assign $E_j \leftarrow E_j^{l}$ (leader), $E_k \leftarrow E_k^{f}$ (follower)
        \STATE $\mathcal{T}_j^{l} \leftarrow \mathcal{T}_j^{l} \cup \{x_n\}$
    \ELSE
        \STATE Assign $E_k \leftarrow E_k^{l}$ (leader), $E_j \leftarrow E_j^{f}$ (follower)
        \STATE $\mathcal{T}_k^{l} \leftarrow \mathcal{T}_k^{l} \cup \{x_n\}$
    \ENDIF
    \STATE Compute expert outputs $y_j = \theta_j x_n$, $y_k = \theta_k x_n$
    \STATE Compute weighted MoE output for $x_n$ using $y_j, y_k$
\ENDFOR
\STATE Compute loss $L$ and backpropagate to obtain output gradients $g_n = \partial L / \partial y_n$ for all tokens
    \FOR{each follower expert $E_k^{f}$ paired with leader $E_j^{l}$} \STATE{// Projection, using buffers \emph{prior} to this step's update}
    \IF{$\mathrm{cossim}(b_j^{(r)}, b_k^{(r)}) > \gamma$}
        \STATE $g_k^{f*} \leftarrow (I - \alpha\, Q_j^\top Q_j)g_k^{f}$
    \ELSE
        \STATE $g_k^{f*} \leftarrow g_k^{f}$
    \ENDIF
\ENDFOR
\FOR{each expert $i$}
    \STATE \COMMENT{Accumulation, updates buffer for next step}
    \IF{$\mathcal{T}_i^{l} \neq \emptyset$}
        \STATE $\bar{g}_i^{l} \leftarrow \frac{1}{|\mathcal{T}_i^{l}|} \sum_{t \in \mathcal{T}_i^{l}} g_t$
        \STATE $b_i^{\text{new}} \leftarrow \mathrm{GramSchmidt}(\bar{g}_i^{l}, Q_i)$
        \STATE Update $Q_i$ with $b_i^{\text{new}}$ (FIFO, evict oldest if $|Q_i| = h$)
    \ENDIF
\ENDFOR
\STATE Update all expert weights using $g_i^l$ (leaders) and $g_i^{f*}$ (followers) via the optimizer
\STATE \textbf{return} $\{\theta_i(t+1)\}$, $\{Q_i\}$
\end{algorithmic}
\end{algorithm}

\section{Proof of Proposition 1}
\label{appendix:proof}

\begin{proof}
Recall the inner product of output space between leader and follower after SGD update:
\begin{equation}
\begin{gathered}
\langle \theta_j^l(t)x - \eta g_j^l(t)x,\,
\theta_k^f(t)x - \eta g_k^{f*}(t)x \rangle \\
= \langle \theta_j^l(t)x, \theta_k^f(t)x \rangle
+ \langle -\eta g_j^l(t)x, \theta_k^f(t)x \rangle
+ \langle -\eta g_k^{f*}(t)x, \theta_j^l(t)x \rangle
+ \eta^2 \langle g_j^l(t)x, g_k^{f*}(t)x \rangle.
\end{gathered}
\label{eq4}
\end{equation}

For $\eta^2 \langle g_j^l(t)x, g_k^{f*}(t)x \rangle$ in \eqref{eq4},    
\begin{align}
⟨\frac{\partial L}{\partial y_j^l} \left\|x\right\|^2,(I-P_j)\frac{\partial L}{\partial y_k^f} \left\|x\right\|^2⟩\notag\\=⟨\frac{\partial L}{\partial y_j^l}x,(I-P_j)\frac{\partial L}{\partial y_k^f}x⟩ \left\|x\right\|^2\approx0\notag
\end{align}
This is because the projection $I-P_j$ made the vector $\frac{\partial L}{\partial y_k^f}x$ orthogonal to $\frac{\partial L}{\partial y_j^l}x$ by eliminating the overlapping direction of it.

The parameters at time $t$ can be expressed by accumulating updates from an initial time $t_0$, where $t_0$ denotes the earliest time at which the short-term gradient consistency assumption holds:
% \begin{align}
% \theta_i(t)=\theta_i(t_0)-\eta\sum_{s=t_0}^{t}g_i(s)\\
% \text{where, } \langle g_i(s-1),g_i(s)\rangle\approx 0, \forall{s<t} \notag
% \end{align}

\begin{align}
\theta_i(t) &= \theta_i(t_0)-\eta\sum_{s=t_0}^{t-1}g_i(s)\\
\text{where }&
\frac{\langle g_i(s-1),g_i(s)\rangle}
{\|g_i(s-1)\|_2\,\|g_i(s)\|_2}
\approx 1,\quad \forall\, t_0 < s \le t-1. \notag
\end{align}

Under this assumption, the accumulated gradients satisfy
\begin{align}
\eta\sum_{s=t_0}^{t-1} g_i(s)
\approx
\alpha_i(t)g_i(t),
\end{align}
where $\alpha_i(t)$ is a scalar representing the accumulated update magnitude. Therefore, the parameter at time $t$ can be expressed as
\begin{align}
\theta_i(t)=\theta(t_0)_i-\alpha_i(t)g_i(t),
\end{align}
where $g_i(t)$ denotes the current unit gradient of expert $i$.

Then, for $\langle -\eta g_k^{f*}(t)x, \theta_j^l(t)x\rangle \notag$ in \eqref{eq4}
\begin{gather}
\langle -\eta g_k^{f*}(t)x, \theta_j^l(t)x\rangle \notag\\
=
\langle-\eta g_k^{f*}(t)x,\theta_j^l(t_0)x\rangle
+\eta\alpha_j(t)\langle g_k^{f*}(t)x,\theta_j^l(t)x\rangle\notag\\
\approx
\langle-\eta g_k^{f*}(t)x,\theta_j^l(t_0)x\rangle
+0.
\end{gather}
where the last approximation follows from the same argument as above.
\end{proof}

Therefore, orthogonalizing the output gradients serves to reduce the alignment between the two different expert output vectors through which a single token passes. However, when the two vectors point in completely opposite directions, this becomes a major obstacle to the training of MoE, which uses a linear combination of outputs from multiple experts as its final output.

\section{Detailed hyperparameter settings}
\label{appedinx:hyperparam}

\subsection{Common Experimental Setting}
\centering

\begin{tabular}{l|c}
  \toprule
  Settings & Value \\
  \hline
  Hidden size & 768 \\
  Number of Experts & 4\\
  Top-k routing & 2\\
  Weight decay & 0.01\\
  Auxiliary loss weight & 0.01 \\
  Optimizer & AdamW(0.9, 0.999)\\
  Train batch size & 32\\
  Initial $\alpha$ & 0.1\\
  Warmup & $0.1\to1.0$ for 1 epoch\\
  Load best model at end & True\\
  Max sequence length & 128 \\
  \bottomrule
\end{tabular}

% \subsection{Fine-tuning Setting}

% Trained with NVIDIA RTX A6000 D6(48GB).

% \centering

% \begin{tabular}{l|c}
%   \toprule
%   Settings & Value \\
%   \hline
%   Learning rate & 2e-5 \\
%   Epochs & 5\\
%   Basis rank & 16 \\
%   FFN intermediate & 3072 \\
%   Eval batch size & 64 \\
%   \bottomrule
% \end{tabular}

% \subsection{From-scratch Setting}

% Trained with NVIDIA RTX A6000 Ada generation(48GB).

% \centering

% \begin{tabular}{l|c}
%   \toprule
%   Settings & Value \\
%   \hline
%   Learning rate & 3e-4 \\
%   Epochs & 50\\
%   Early stopping patience & 3 \\
%   Basis rank & 8 \\
%   FFN intermediate & 2048 \\
%   Eval batch size & 32 \\
%   \bottomrule
% \end{tabular}

\subsection{Experimental Settings}

\noindent
\begin{minipage}[c]{0.48\textwidth}
\centering
\textbf{Fine-tuning}\\[6pt]
NVIDIA RTX A6000 D6 (48GB)\\[10pt]

\begin{tabular}{l|c}
  \toprule
  Settings & Value \\
  \midrule
  Learning rate & 2e-5 \\
  Epochs & 5 \\
  Basis rank & 16 \\
  FFN intermediate & 3072 \\
  Eval batch size & 64 \\
  \bottomrule
\end{tabular}
\end{minipage}
\hfill
\begin{minipage}[c]{0.48\textwidth}
\centering
\textbf{Training from Scratch}\\[6pt]
NVIDIA RTX A6000 Ada generation (48GB)\\[10pt]

\begin{tabular}{l|c}
  \toprule
  Settings & Value \\
  \midrule
  Learning rate & 3e-4 \\
  Epochs & 50 \\
  Early stopping patience & 3 \\
  Basis rank & 8 \\
  FFN intermediate & 2048 \\
  Eval batch size & 32 \\
  \bottomrule
\end{tabular}
\end{minipage}

\section{Expert Output Angles Across All Layers}
\label{appedinx:angle}

Fine-tuned models with QNLI, the angle of expert outputs of each layer. The input token is \textit{organization} of the input sequence \textit{What organization is devoted to Jihad against Israel?}, from QNLI test set. Output after 3 epochs, which was the best model for all.

\begin{center}
    \includegraphics[width=0.9\linewidth]{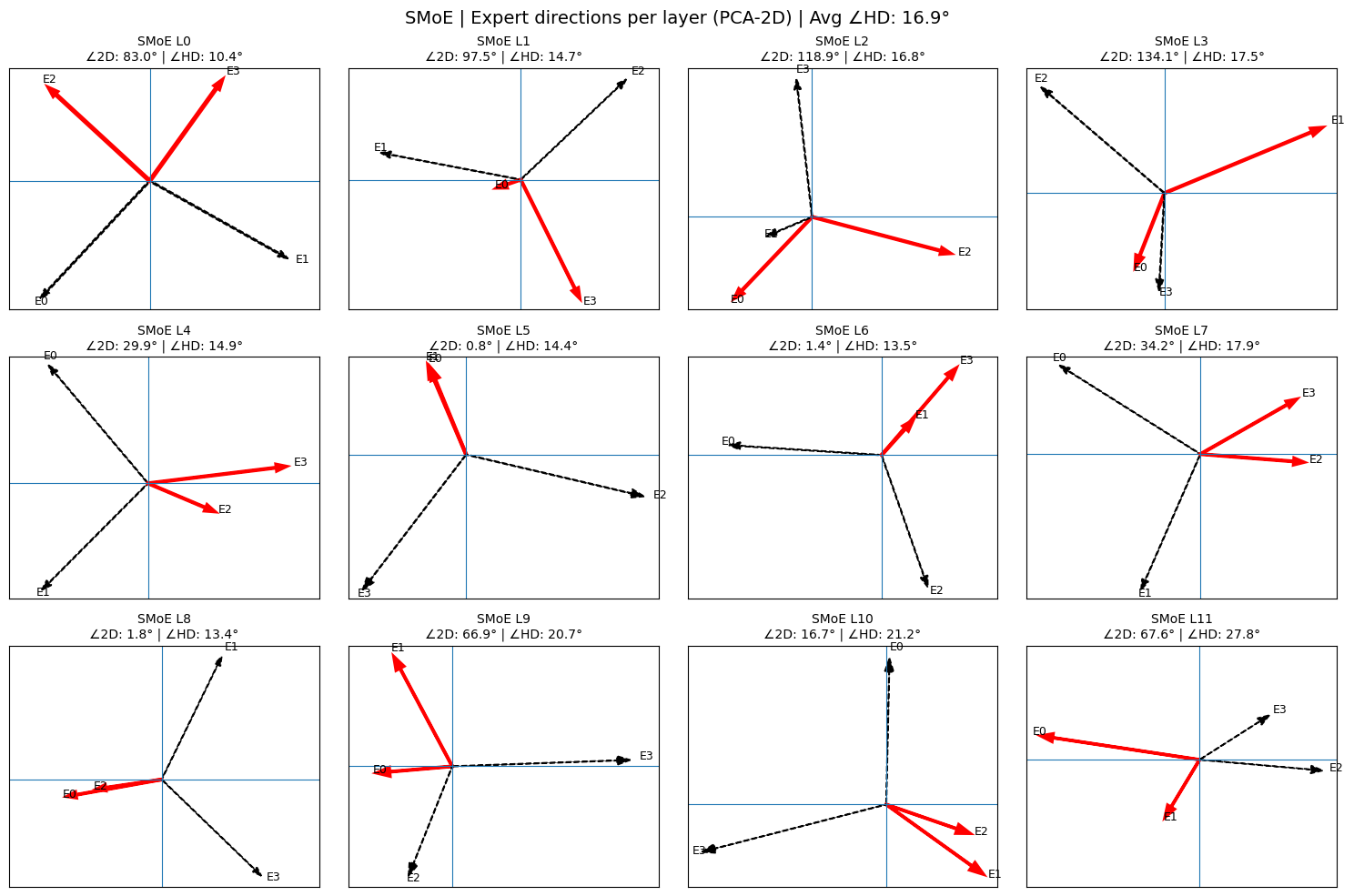}

    \small Angle visualization of SMoE
\end{center}

\begin{center}
    \includegraphics[width=0.9\linewidth]{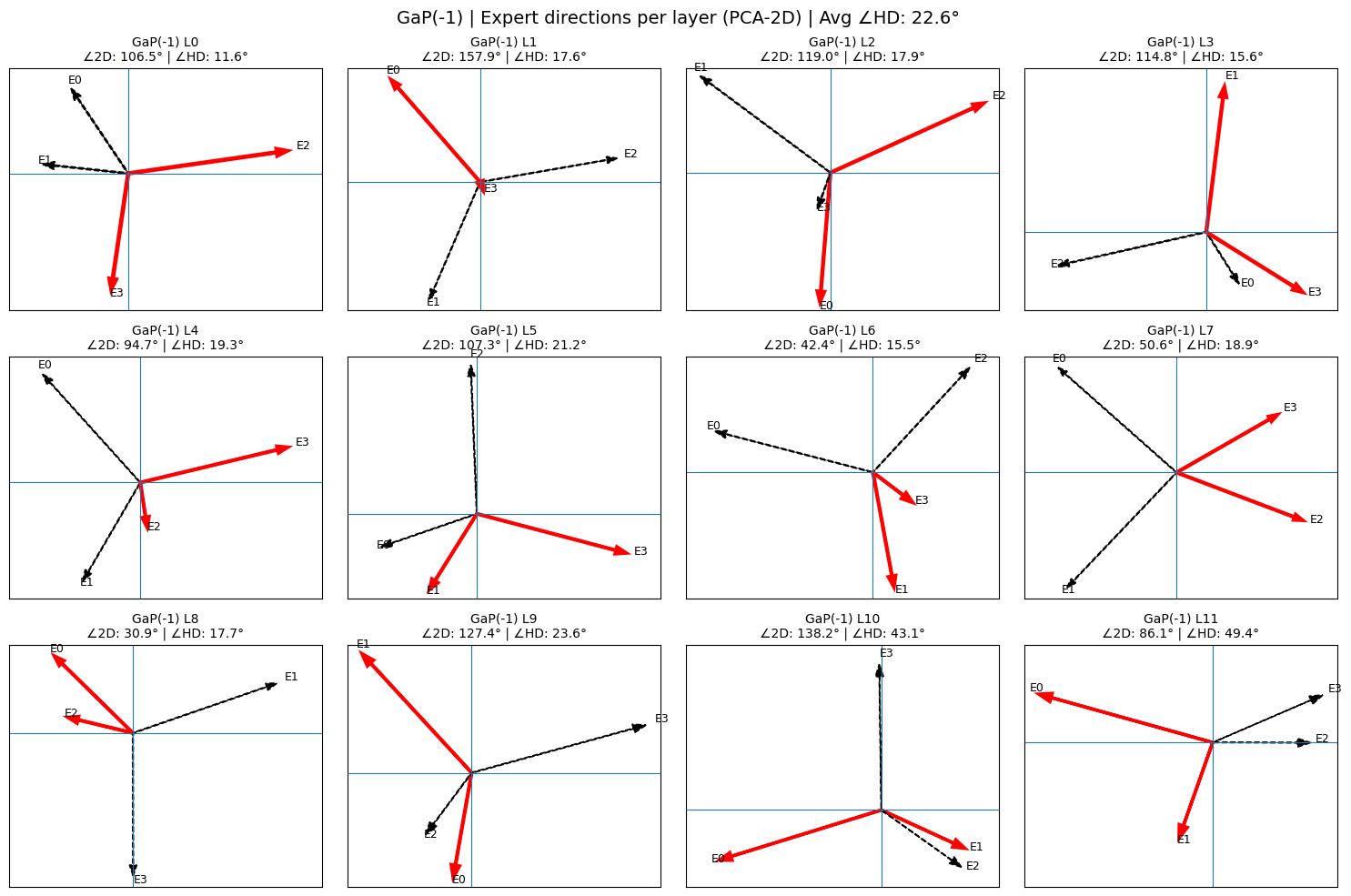}

    \small Angle visualization of GaP-MoE(-1)
\end{center}

\begin{center}
    \includegraphics[width=0.9\linewidth]{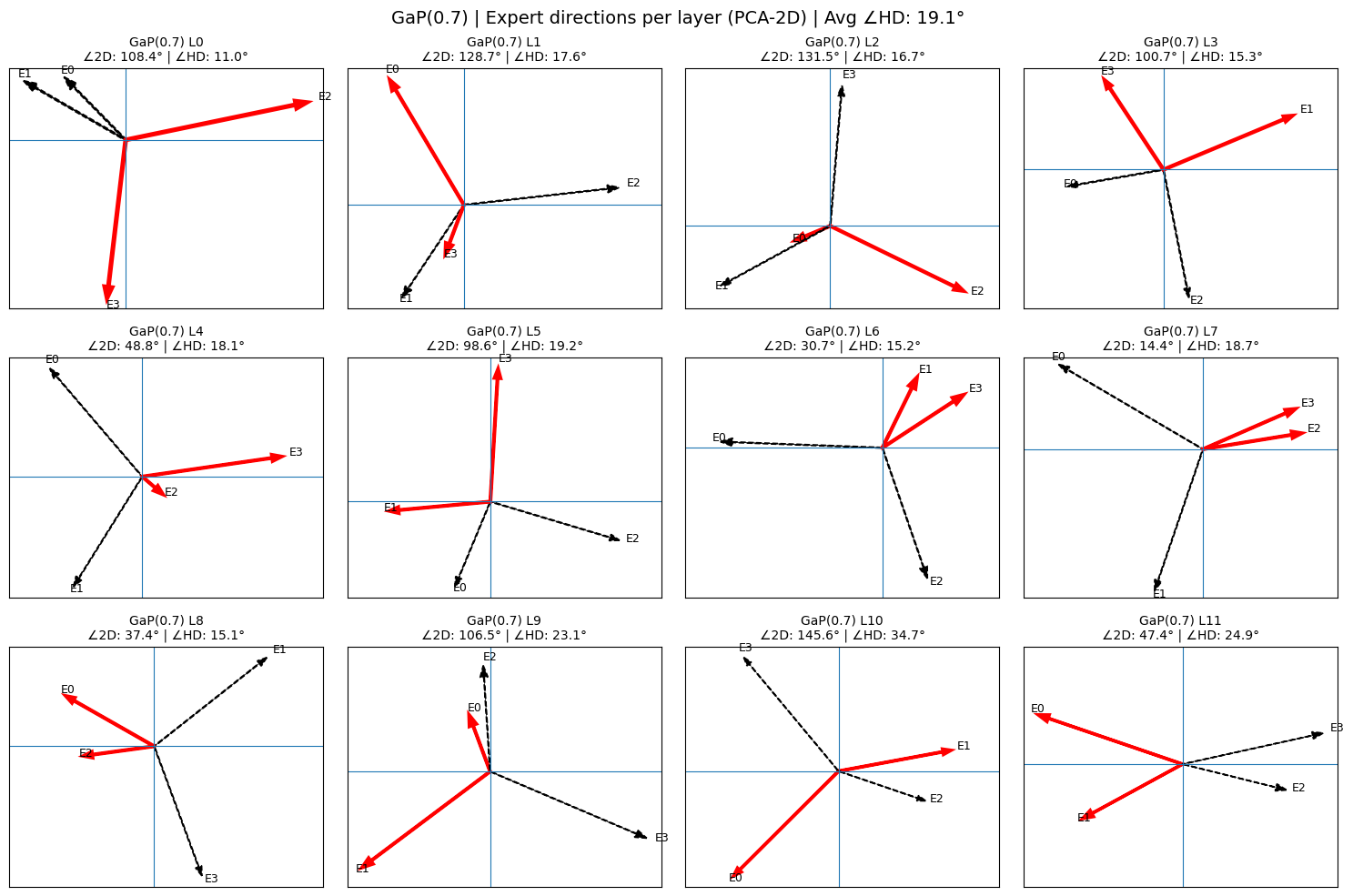}

    \small Angle visualization of GaP-MoE(0.7)
\end{center}

\end{document}